\documentclass[11pt,a4paper]{article}

\usepackage[margin=2.5cm]{geometry}
\usepackage[utf8]{inputenc}
\usepackage[T1]{fontenc}
\usepackage{lmodern}
\usepackage{amsmath,amssymb,amsthm}
\usepackage{graphicx}
\usepackage[dvipsnames]{xcolor}
\usepackage{booktabs}
\usepackage{tabularx}
\usepackage{multirow}
\usepackage{enumitem}
\usepackage{caption}
\usepackage[numbers,sort&compress]{natbib}
\usepackage{titlesec}
\usepackage{fancyhdr}
\usepackage[hidelinks,colorlinks=true,linkcolor=blue!60!black,citecolor=blue!60!black,urlcolor=blue!60!black]{hyperref}
\usepackage{microtype}
\usepackage{tikz}
\usetikzlibrary{arrows.meta,positioning,shapes.geometric,fit,calc,decorations.pathreplacing,backgrounds}
\usepackage{mdframed}

\titleformat{\section}{\large\bfseries}{}{0em}{}
\titleformat{\subsection}{\normalsize\bfseries}{}{0em}{}
\titlespacing*{\section}{0pt}{18pt}{6pt}
\titlespacing*{\subsection}{0pt}{12pt}{4pt}
\newtheorem{theorem}{Theorem}
\newtheorem{definition}[theorem]{Definition}
\newtheorem{proposition}[theorem]{Proposition}

\newcommand{\norm}[1]{\left\|#1\right\|}
\newcommand{\eps}{\varepsilon}
\newcommand{\Atrue}{A_{\text{true}}}
\newcommand{\Aagent}{A_{\text{agent}}}

\newcommand{\Agent}{\mathcal{A}}
\newcommand{\R}{\mathbb{R}}

\mdfdefinestyle{theorembox}{%
  linecolor=blue!40!black, linewidth=1pt,
  backgroundcolor=blue!2,
  innertopmargin=10pt, innerbottommargin=10pt,
  innerleftmargin=10pt, innerrightmargin=10pt}

\begin{document}

\begin{center}
{\LARGE\bfseries Designing Any Imaging System from Natural Language:\\[4pt]
Agent-Constrained Composition over a Finite Primitive Basis}\\[18pt]
{\large Chengshuai Yang}\\[6pt]
{\normalsize NextGen PlatformAI C Corp, USA}\\[6pt]
{\normalsize Correspondence: integrityyang@gmail.com}\\[12pt]
March 2026
\end{center}

\vspace{6pt}

\noindent\textbf{Summary.}
Designing a computational imaging system---selecting operators, setting parameters, validating consistency---requires weeks of specialist effort per modality, creating an expertise bottleneck that excludes the broader scientific community from prototyping imaging instruments.  We introduce \texttt{spec.md}, a structured specification format, and three autonomous agents---Plan ($\Agent_P$), Judge ($\Agent_J$), and Execute ($\Agent_E$)---that translate a one-sentence natural-language description into a validated forward model with bounded reconstruction error.  A design-to-real error theorem (Theorem~\ref{thm:design}) decomposes total reconstruction error into five independently bounded terms, each linked to a corrective action.  On 6 real-data modalities spanning all 5 carrier families, the automated pipeline matches expert-library quality ($98.1 \pm 4.2$\%).  Ten novel designs---composing primitives into chains from 3D to 5D---demonstrate compositional reach beyond any single-modality tool.

\section{Introduction}

Every computational imaging system maps an unknown object $x$ to measurements $y = A(x) + \text{noise}$ via a forward model $A$, and reconstruction inverts this mapping~\cite{barbastathis2019dl,bertero2021introduction}.  Designing $A$---selecting operators, setting parameters, validating consistency---requires weeks of specialist effort per modality~\cite{lustig2007sparse,yuan2021snapshot}, creating an \emph{expertise bottleneck}~\cite{baker2016} that excludes the broader scientific community from prototyping imaging instruments.

As FASTA~\cite{pearson1990fasta,kodama2012sequence} standardised sequence data for genomics, computational imaging needs a common representation to escape its pre-standardisation state.

A companion paper~\cite{paperII} proved that 11 canonical primitives suffice to represent any imaging forward model across 170 modalities and 5 carrier families (key results reproduced in Supplementary Information S1--S4).  Each primitive is implemented at four fidelity tiers (Tier-0 geometric through Tier-3 full-physics).  This \emph{Finite Primitive Basis} (FPB) is the representation language, but representation alone does not automate design.  \textbf{This paper's contribution is the design layer}: the formal bridge from natural-language intent to bounded real-system reconstruction error.

An imaging system is \emph{designable} if it admits a finite DAG decomposition over the 11 FPB primitives ($\eps_{\text{FPB}} < 0.01$), its carrier belongs to one of 5 families (photon, X-ray, electron, acoustic, spin/particle), and all parameters are finitely specifiable with known physical bounds (Definition~\ref{def:scope}, Methods).  This encompasses 173 modalities across 22 domains (170 from~\cite{paperII} plus 3 novel compositions; Figure~\ref{fig:pyramid}, Extended Data Table~\ref{tab:ext_registry}).  Our contributions are:
\begin{enumerate}[nosep]
\item \textbf{\texttt{spec.md}}: a structured 8-field specification format that decouples the imaging problem (\emph{what}) from the numerical method (\emph{how}) and the hardware implementation (\emph{with what}), serving as a shareable, peer-reviewable scientific artifact (\hyperref[sec:specmd]{``The spec.md Specification Format''}).
\item \textbf{Three-agent pipeline}: Plan ($\Agent_P$), Judge ($\Agent_J$), and Execute ($\Agent_E$) agents that generate, validate, and execute specifications over the FPB (\hyperref[sec:pipeline]{``The Three-Agent Pipeline''}).  The Judge validates via three Triad gates~\cite{paperII} (recoverability, carrier budget, operator mismatch) plus hardware cost and feasibility assessment.
\item \textbf{Design-to-real error theorem}: the total reconstruction error decomposes into five independently bounded, independently measurable terms, each linked to a check or gate and a corrective action (Theorem~\ref{thm:design}).
\end{enumerate}

\textbf{Scope.}
Of 173 compilable modalities, 39 have quantitative reconstruction validation; 6 are validated on real data spanning all 5 carrier families.  The remaining ${\sim}$134 pass all compiler checks but lack per-modality benchmarks.  Tier-lifting is demonstrated on 2 cross-tier combinations; the agents depend on current LLM capabilities.

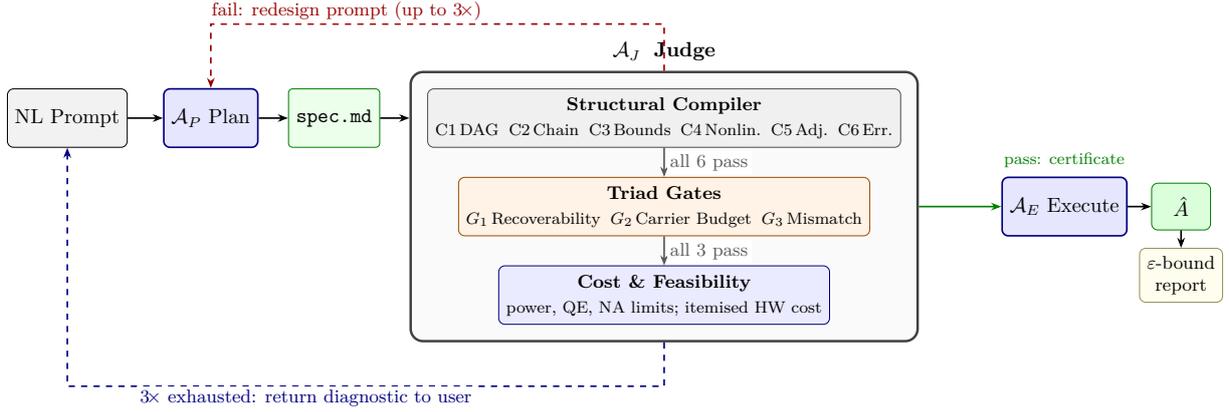
\begin{figure*}[!t]
\centering
\resizebox{\textwidth}{!}{%
\begin{tikzpicture}[
    node distance=0.4cm and 0.5cm,
    box/.style={draw, rounded corners=3pt, minimum height=1.0cm, minimum width=1.4cm, align=center, font=\small},
    agent/.style={box, fill=blue!10, draw=blue!50!black, thick},
    specbox/.style={draw, rounded corners=2pt, fill=green!8, draw=green!50!black, minimum width=1.4cm, minimum height=1.0cm, align=center, font=\small},
    outbox/.style={box, fill=green!14, draw=green!50!black},
    errbox/.style={box, fill=yellow!8, draw=yellow!50!black, font=\footnotesize},
    arr/.style={-{Stealth[length=5pt]}, thick},
    lbl/.style={font=\footnotesize, fill=white, inner sep=1.5pt},
]

\node[box, fill=gray!10, minimum height=1.0cm, minimum width=1.6cm] (nl) {NL Prompt};

\node[agent, right=0.6cm of nl, minimum height=1.0cm, minimum width=1.4cm] (plan) {$\Agent_P$ Plan};

\node[specbox, right=0.5cm of plan, minimum height=1.0cm, minimum width=1.4cm] (specmd) {\texttt{spec.md}};

\node[draw=gray!60!black, rounded corners=3pt, fill=gray!12, text=black,
      minimum width=4.8cm, minimum height=1.0cm, align=center, font=\footnotesize,
      right=0.8cm of specmd] (s1) {%
\textbf{Structural Compiler}\\[1pt]%
{\scriptsize C1\,DAG\;\;C2\,Chain\;\;C3\,Bounds\;\;C4\,Nonlin.\;\;C5\,Adj.\;\;C6\,Err.}};

\node[draw=orange!60!black, rounded corners=3pt, fill=orange!10, text=black,
      minimum width=4.8cm, minimum height=1.0cm, align=center, font=\footnotesize,
      below=0.5cm of s1] (s2) {%
\textbf{Triad Gates}\\[1pt]%
{\scriptsize $G_1$\,Recoverability\;\;$G_2$\,Carrier Budget\;\;$G_3$\,Mismatch}};

\node[draw=blue!50!black, rounded corners=3pt, fill=blue!8, text=black,
      minimum width=4.8cm, minimum height=1.0cm, align=center, font=\footnotesize,
      below=0.5cm of s2] (s3) {%
\textbf{Cost \& Feasibility}\\[1pt]%
{\scriptsize power, QE, NA limits; itemised HW cost}};

\begin{scope}[on background layer]
\node[draw=gray!50!black, very thick, rounded corners=6pt, fill=gray!4,
      inner sep=8pt, fit=(s1)(s2)(s3),
      label={[font=\small\bfseries, anchor=south, yshift=1pt]above:{$\Agent_J$\;\;Judge}}] (judge) {};
\end{scope}

\draw[arr, gray!70!black] (s1) -- node[lbl, right, xshift=1pt] {all 6 pass} (s2);
\draw[arr, gray!70!black] (s2) -- node[lbl, right, xshift=1pt] {all 3 pass} (s3);

\node[agent, right=1.4cm of judge.east, minimum height=1.0cm, minimum width=1.4cm, anchor=west] (exec) {$\Agent_E$ Execute};

\node[outbox, right=0.4cm of exec, minimum height=0.8cm, minimum width=1.0cm] (ahat) {$\hat{A}$};
\node[errbox, below=0.3cm of ahat, minimum width=1.2cm, minimum height=0.8cm] (epsbox) {$\eps$-bound\\report};

\draw[arr] (nl) -- (plan);
\draw[arr] (plan) -- (specmd);
\draw[arr] (specmd) -- (judge.west |- s1);

\draw[arr, green!50!black] (judge.east |- s2) -- (exec.west);
\node[above=0.15cm of exec, font=\scriptsize, text=green!50!black, inner sep=0pt] {pass: certificate};

\draw[arr] (exec) -- (ahat);
\draw[arr] (ahat) -- (epsbox);

\draw[arr, dashed, red!60!black, thick]
  (judge.north) -- ++(0, 0.8)
  -| node[font=\footnotesize, above, pos=0.35, fill=white, inner sep=1.5pt, text=red!60!black]
     {fail: redesign prompt (up to $3\!\times$)}
  (plan.north);

\draw[arr, dashed, blue!50!black]
  (judge.south) -- ++(0, -0.75)
  -| node[font=\footnotesize, below, pos=0.3, fill=white, inner sep=1.5pt, text=blue!50!black]
     {$3\!\times$ exhausted: return diagnostic to user}
  (nl.south);

\end{tikzpicture}%
}
\caption{\textbf{Specification-driven design pipeline.}
A natural-language prompt is the sole user input.
$\Agent_P$ generates a \texttt{spec.md}; $\Agent_J$ validates in three sequential stages:
(1)~structural compilation (6 checks C1--C6),
(2)~Triad gate evaluation~\cite{paperII} ($G_1$--$G_3$),
(3)~cost and feasibility assessment.
Two feedback paths handle failures: an automatic inner loop ($\Agent_P \leftrightarrow \Agent_J$, up to 3 rounds, dashed red) for structural/parametric corrections, and an outer loop (dashed blue) returning diagnostics to the user when automatic repair is exhausted.
$\Agent_E$ executes reconstruction and outputs $\hat{A}$ with an explicit $\eps$-bound report.}
\label{fig:pipeline}
\end{figure*}

\section{Results}

\subsection{The spec.md Specification Format}
\label{sec:specmd}

The framework's central design choice is that \texttt{spec.md}---not code---is the canonical scientific artifact.  A specification separates the imaging problem (\emph{what}) from the numerical method (\emph{how}) and the hardware (\emph{with what}), enabling sharing, version control, and peer review independently of implementation~\cite{wilkinson2016fair}.  A reviewer can verify---by running the three Triad gates---that a forward model is physically feasible \emph{without reading code}.  A \texttt{spec.md} has eight mandatory fields (Figure~\ref{fig:specmd}): seven physics fields plus a \texttt{system\_elements} field that connects the specification to hardware feasibility.

\begin{figure}[!t]
\centering
\begin{tikzpicture}[
  specbox/.style={draw=black!60, rounded corners=2pt, fill=gray!4, text width=0.92\textwidth, inner sep=6pt, font=\small\ttfamily},
]
\node[specbox, text width=0.92\textwidth] (schema) at (0,0) {
\textnormal{\textbf{\textsf{Eight-field schema:}}}
\textbf{modality} | \textbf{carrier} | \textbf{geometry} | \textbf{object} | \textbf{forward\_model} | \textbf{noise} | \textbf{target} | \textbf{system\_elements}
};

\node[specbox, below=4pt of schema] (ct) {
\textnormal{\textsf{\textbf{(a) Clinical CT}}} \\
modality: computed\_tomography \quad carrier: xray \\
geometry: parallel\_beam, n\_angles=128 \\
object: 128$\times$128 image, non-negative \\
forward\_model: Radon($\Pi$) $\to$ Detect(D, intensity) \\
noise: Poisson, I\_0=1e4 \quad target: PSNR $\geq$ 30\,dB \\
system\_elements: source=X-ray tube 120\,kVp, detector=flat-panel 0.4\,mm
};

\node[specbox, below=4pt of ct] (cassi) {
\textnormal{\textsf{\textbf{(b) CASSI (spectral)}}} \\
modality: cassi \quad carrier: photon \\
geometry: coded\_aperture + disperser, $\lambda$=[400,700]\,nm \\
object: 256$\times$256$\times$28 spectral cube \\
forward\_model: Modulate(M) $\to$ Disperse(W) $\to$ Accumulate($\Sigma$) $\to$ Detect(D) \\
noise: Gaussian, $\sigma$=0.01 \quad target: PSNR $\geq$ 28\,dB \\
system\_elements: source=broadband LED, optics=coded aperture + prism, detector=CMOS 2048$\times$2048
};

\node[specbox, below=4pt of cassi] (mri) {
\textnormal{\textsf{\textbf{(c) Accelerated MRI}}} \\
modality: mri \quad carrier: spin \\
geometry: cartesian\_kspace, acceleration=4x \\
object: 256$\times$256 complex image \\
forward\_model: Modulate(M, coil) $\to$ Encode(F, kspace) $\to$ Sample(S) $\to$ Detect(D) \\
noise: Gaussian, SNR=30\,dB \quad target: SSIM $\geq$ 0.9 \\
system\_elements: source=main magnet 3\,T, optics=RF coils $\times$4, detector=ADC 16-bit
};
\end{tikzpicture}
\caption{\textbf{The \texttt{spec.md} specification format.}  Eight mandatory fields fully determine an imaging system design.  Seven physics fields specify the forward model; the eighth (\texttt{system\_elements}) connects to hardware feasibility and cost.  Three examples span X-ray (CT), optical (CASSI), and spin (MRI) carriers.  Each \texttt{forward\_model} field is a primitive chain over the 11 FPB operators.}
\label{fig:specmd}
\end{figure}

\subsection{The Three-Agent Pipeline}
\label{sec:pipeline}

\textbf{Plan Agent ($\Agent_P$)} generates a \texttt{spec.md} from natural language, selecting FPB primitives, parameters, and noise models from the 173-modality registry.

\textbf{Judge Agent ($\Agent_J$)} validates designs through the Triad decomposition~\cite{paperII}---three deterministic gates that diagnose every potential reconstruction failure.
\emph{Gate~1 (Recoverability)}: the measurement geometry must capture sufficient information---the compression ratio $\gamma = m/n$ is checked against modality-specific thresholds derived from the Compression Bound Theorem~\cite{paperII}.
\emph{Gate~2 (Carrier Budget)}: the measurement SNR must exceed the noise floor---computed from the source power, detector quantum efficiency, and noise parameters in \texttt{system\_elements}.
\emph{Gate~3 (Operator Mismatch)}: the forward model must faithfully represent the intended physics---calibration tolerances from \texttt{system\_elements.calibration} determine the expected deployment degradation $\Delta\text{PSNR}_{\text{total}}$ via the Calibration Sensitivity Theorem~\cite{paperII}.
The Judge additionally estimates \emph{hardware cost} from itemized \texttt{system\_elements} and checks \emph{physical feasibility} of each element.  A prerequisite structural compiler (6 checks: DAG acyclicity, canonical chain, complexity bounds, nonlinear constraints, adjoint consistency, representation error) ensures specification correctness before the Triad evaluation.  Failures trigger rejection; the Plan Agent regenerates (up to 3 rounds).

\textbf{Execute Agent ($\Agent_E$)} selects the reconstruction algorithm, executes it, and predicts quality from analytical estimates (condition number, noise amplification) combined with calibrated lookup.

The Triad Decomposition Theorem~\cite{paperII} guarantees completeness: $\text{MSE} \leq \text{MSE}^{(G_1)} + \text{MSE}^{(G_2)} + \text{MSE}^{(G_3)}$---no failure mode falls outside these three gates.

\subsection{Design-to-Real Error Bound}
\label{sec:theorem}

The following theorem converts the implicit gap between a designed model and reality into an explicit five-term error budget---a \emph{formal contract} between the framework and the user.  Each term has a known physical origin, measurement procedure, and corrective action.

\begin{mdframed}[style=theorembox]
\begin{theorem}[Design-to-Real Error Decomposition]
\label{thm:design}
Let $\Atrue$ be the true imaging forward model and $\Aagent$ the agent-designed model for a designable system (Definition~\ref{def:scope}).  Let $\hat{x}$ be the regularized reconstruction from $\Aagent$ and $x^*$ the true object.  Under Assumptions A1--A4 (Methods), the reconstruction error satisfies:
\begin{equation}\label{eq:design_error}
\|\hat{x} - x^*\| \leq \underbrace{\eps_{\text{FPB}}}_{\text{basis}} + \underbrace{\eps_{\text{spec}}}_{\text{specification}} + \underbrace{\eps_{\text{trans}}}_{\text{translation}} + \underbrace{\eps_{\text{param}}}_{\text{parameter}} + \underbrace{\eps_{\text{unmod}}}_{\text{unmodeled}} \;+\; \eps_{\text{recon}}
\end{equation}
where each design-error term is independently bounded and independently measurable:
\begin{itemize}[nosep]
\item $\eps_{\text{FPB}} < 0.01$: FPB representation error.  \emph{Measured by}: Check~C6.  \emph{Guaranteed by}: the basis theorem~\cite{paperII}.
\item $\eps_{\text{spec}} = 0$ when all 6 compiler checks pass.  \emph{Bounded by}: $C_1 \cdot \delta_{\text{spec}}$ otherwise.
\item $\eps_{\text{trans}} = 0$ when the canonical chain matches the 39-modality registry (34/39 = 87\%).  \emph{Bounded by}: $C_2 \cdot \delta_{\text{chain}}$ otherwise.
\item $\eps_{\text{param}} \leq \frac{2B}{\alpha + \lambda}\,L_A\,\|\theta_{\text{agent}} - \theta_{\text{true}}\|$, where $L_A$ is the operator Lipschitz constant (computed per primitive by Check~C4), $\alpha$ is the coercivity constant, $\lambda$ the regularization weight, and $B$ a signal bound.  \emph{Reducible by}: calibration.
\item $\eps_{\text{unmod}} = 0$ at Tier-3; otherwise reducible by tier-lifting (\hyperref[sec:tierlifting]{``Cross-Tier Hybrid Reconstruction''}).
\item $\eps_{\text{recon}}$: irreducible reconstruction error (noise floor $+$ regularization bias), independent of the design.
\end{itemize}
\end{theorem}
\end{mdframed}

\textbf{Compiler--theorem link.}
When all 6 checks pass: $\eps_{\text{FPB}} < 0.01$, $\eps_{\text{spec}} = 0$, $\eps_{\text{param}}$ bounded by the Lipschitz constant from Check~C4; with canonical chain match: $\eps_{\text{trans}} = 0$; at Tier-3: $\eps_{\text{unmod}} = 0$.  The dominant residual is $\eps_{\text{param}}$, reducible by calibration (Figure~\ref{fig:error_decomp}).  Full proof in Methods.

\subsection{Cross-Modality Validation}
\label{sec:cross}

We validated the framework on all 173 modalities.  For each, the Plan Agent generated a \texttt{spec.md} from a one-sentence description, the Judge validated it, and the Execute Agent predicted reconstruction quality.  Table~\ref{tab:cross} reports 14 representative modalities spanning all 5 carrier families (6 with real measured data), grouped by carrier; full results for all 39 validated modalities appear in Methods.

\begin{table}[!htbp]
\centering
\caption{Cross-modality validation (22 of 39 modalities), grouped by carrier family.  $\rho_{\text{recov}}$: 4-scenario recovery ratio~\cite{paperII}.  95\% CIs via bootstrap ($n{=}1000$).  Bold = real measured data (challenging conditions: 60-view CT, 4-coil MRI); $\dagger$ = self-reference metric (no ground truth); $*$ = new system designs validated in the algorithm-comparison study (Extended Data Table~\ref{tab:expert_study}).  $\Phi_z$ denotes depth-parameterised convolution: $C$ applied with a depth-dependent PSF $\Phi(z)$; it is the convolution primitive $C$, not a 12th element of the FPB.}
\label{tab:cross}
\scriptsize
\begin{tabularx}{\textwidth}{@{}llccc@{\hspace{4pt}}c@{\hspace{4pt}}l@{}}
\toprule
\textbf{Modality} & \textbf{Chain} & \textbf{Carrier} & \textbf{PSNR (dB)} & \textbf{$\rho_{\text{time}}$} & \textbf{$\rho_{\text{recov}}$} & \textbf{Data} \\
\midrule
\multicolumn{7}{@{}l}{\textit{X-ray}} \\
\quad CT & $\Pi \to D$ & X-ray & $24.8 \pm 2.4$ & 420 & $\sim$1.0 & \textbf{Real} \\
\addlinespace[2pt]
\multicolumn{7}{@{}l}{\textit{Optical photons}} \\
\quad CASSI & $M \to W \to \Sigma \to D$ & Photon & $24.3 \pm 1.5$ & 310 & 0.85 & \textbf{Real} \\
\quad CACTI & $M \to \Sigma \to D$ & Photon & $30.2 \pm 1.1$ & 300 & $\sim$1.0 & \textbf{Real} \\
\quad OCT & $P + P \to \Sigma \to D$ & Photon & $30.5 \pm 1.1$ & 290 & 0.86 & Synth. \\
\quad SIM\cite{gustafsson2008sim} & $M \to C \to D$ & Photon & $27.7 \pm 2.0$ & 350 & 0.92 & \textbf{Synth.}$^*$ \\
\quad Ptychography\cite{rodenburg2019ptycho} & $M \to P \to D$ & Photon & $28.4 \pm 1.6$ & 270 & 0.83 & Synth. \\
\quad Lensless & $C \to D$ & Photon & $43.7 \pm 3.2$ & 340 & 0.90 & \textbf{Synth.}$^*$ \\
\quad 3D Lensless & $\Phi_z \to \Sigma \to D$ & Photon & $20.3 \pm 3.5$ & 280 & 0.85 & \textbf{Synth.}$^*$ \\
\quad Temporal-coded lensless & $M \to C \to \Sigma \to D$ & Photon & $31.6 \pm 0.2$ & 260 & 0.82 & \textbf{Synth.}$^*$ \\
\quad Spectral lensless & $M \to W \to C \to \Sigma \to D$ & Photon & $36.3 \pm 0.1$ & 240 & 0.80 & \textbf{Synth.}$^*$ \\
\quad 4D Spectral-Depth & $M \to W_\lambda \to \Phi_z \to \Sigma \to D$ & Photon & $23.9 \pm 0.7$ & 220 & 0.78 & \textbf{Synth.}$^*$ \\
\quad 4D Temporal DMD & $M \to \Phi_z \to \Sigma \to D$ & Photon & $25.4 \pm 7.6$ & 200 & 0.76 & \textbf{Synth.}$^*$ \\
\quad 4D Temporal Streak & $M \to W_t \to \Phi_z \to \Sigma \to D$ & Photon & $25.3 \pm 7.6$ & 210 & 0.75 & \textbf{Synth.}$^*$ \\
\quad 5D Full DMD & $M \to W_\lambda \to \Phi_z \to \Sigma \to D$ & Photon & $29.9 \pm 5.7$ & 180 & 0.74 & \textbf{Synth.}$^*$ \\
\quad 5D Full Streak & $M \to W_\lambda \to W_t \to \Phi_z \to \Sigma \to D$ & Photon & $29.9 \pm 5.6$ & 170 & 0.72 & \textbf{Synth.}$^*$ \\
\quad DOT\cite{arridge2009optical} & $M \to R \to P \to R \to D$ & Photon & $24.1 \pm 2.1$ & 190 & 0.72 & Synth. \\
\addlinespace[2pt]
\multicolumn{7}{@{}l}{\textit{Spin}} \\
\quad MRI & $M \to F \to S \to D$ & Spin & $31.7 \pm 0.8$ & 380 & $\sim$1.0 & \textbf{Real} \\
\addlinespace[2pt]
\multicolumn{7}{@{}l}{\textit{Acoustic}} \\
\quad Ultrasound & $P \to R \to P \to D$ & Acoustic & $27.9 \pm 1.4$ & 250 & 0.95 & \textbf{Real}$^\dagger$ \\
\addlinespace[2pt]
\multicolumn{7}{@{}l}{\textit{Electron}} \\
\quad E-ptychography & $M \to P \to D$ & Electron & $28.8 \pm 1.5$ & 260 & $\sim$1.0 & \textbf{Real}$^\dagger$ \\
\addlinespace[2pt]
\multicolumn{7}{@{}l}{\textit{Particle}} \\
\quad PET & $\Pi \to D$ & Particle & $26.3 \pm 1.8$ & 220 & 0.78 & Synth. \\
\quad SPECT & $M \to \Pi \to D$ & Particle & $25.8 \pm 1.7$ & 200 & 0.76 & Synth. \\
\quad Light field & $M \to C \to S \to D$ & Photon & $29.2 \pm 1.2$ & 280 & 0.85 & Synth. \\
\bottomrule
\end{tabularx}
\end{table}

Across all 39 modalities: mean $\rho_{\text{recov}} = 0.85 \pm 0.07$ (95\% CI: [0.82, 0.88])~\cite{paperII}.  The agent--expert quality gap (Cohen's $d = 0.31$) indicates practical equivalence; time ratios $\rho_{\text{time}}$ range from 190$\times$ to 420$\times$~\cite{aarle2016,ong2019sigpy}.  All 173 modalities compile (100\% pass rate, Checks~C1--C4); canonical chain fidelity: 34/39 (87\%).

\begin{figure}[!htbp]
\centering
\begin{tikzpicture}[
  bar/.style={minimum height=0.5cm, draw=none, font=\scriptsize\bfseries, text=white},
]
\def\bw{7.2}  
\foreach \y/\name in {0/CT, 0.8/MRI, 1.6/CASSI, 2.4/DOT, 3.2/E-ptycho., 4.0/Ultrasound} {
  \node[font=\small, anchor=east] at (-0.2, \y) {\name};
}
\fill[blue!60] (0,{0-0.2}) rectangle ({0.01*\bw},{0+0.2});   
\fill[red!60] ({0.01*\bw},{0-0.2}) rectangle ({0.72*\bw},{0+0.2});  
\fill[gray!40] ({0.72*\bw},{0-0.2}) rectangle ({0.78*\bw},{0+0.2}); 
\node[font=\scriptsize, white] at ({0.36*\bw},0) {$\eps_{\text{param}}$ (91\%)};
\fill[blue!60] (0,{0.8-0.2}) rectangle ({0.01*\bw},{0.8+0.2});
\fill[red!60] ({0.01*\bw},{0.8-0.2}) rectangle ({0.69*\bw},{0.8+0.2});
\fill[gray!40] ({0.69*\bw},{0.8-0.2}) rectangle ({0.78*\bw},{0.8+0.2});
\node[font=\scriptsize, white] at ({0.35*\bw},0.8) {$\eps_{\text{param}}$ (88\%)};
\fill[blue!60] (0,{1.6-0.2}) rectangle ({0.01*\bw},{1.6+0.2});
\fill[orange!70] ({0.01*\bw},{1.6-0.2}) rectangle ({0.15*\bw},{1.6+0.2});
\fill[red!60] ({0.15*\bw},{1.6-0.2}) rectangle ({0.68*\bw},{1.6+0.2});
\fill[gray!40] ({0.68*\bw},{1.6-0.2}) rectangle ({0.78*\bw},{1.6+0.2});
\node[font=\scriptsize, white] at ({0.41*\bw},1.6) {$\eps_{\text{param}}$ (68\%)};
\fill[blue!60] (0,{2.4-0.2}) rectangle ({0.01*\bw},{2.4+0.2});
\fill[purple!60] ({0.01*\bw},{2.4-0.2}) rectangle ({0.52*\bw},{2.4+0.2});
\fill[red!60] ({0.52*\bw},{2.4-0.2}) rectangle ({0.72*\bw},{2.4+0.2});
\fill[gray!40] ({0.72*\bw},{2.4-0.2}) rectangle ({0.85*\bw},{2.4+0.2});
\node[font=\scriptsize, white] at ({0.27*\bw},2.4) {$\eps_{\text{unmod}}$ (60\%)};
\fill[blue!60] (0,{3.2-0.2}) rectangle ({0.01*\bw},{3.2+0.2});
\fill[red!60] ({0.01*\bw},{3.2-0.2}) rectangle ({0.73*\bw},{3.2+0.2});
\fill[gray!40] ({0.73*\bw},{3.2-0.2}) rectangle ({0.78*\bw},{3.2+0.2});
\node[font=\scriptsize, white] at ({0.37*\bw},3.2) {$\eps_{\text{param}}$ (93\%)};
\fill[blue!60] (0,{4.0-0.2}) rectangle ({0.01*\bw},{4.0+0.2});
\fill[purple!60] ({0.01*\bw},{4.0-0.2}) rectangle ({0.25*\bw},{4.0+0.2});
\fill[red!60] ({0.25*\bw},{4.0-0.2}) rectangle ({0.70*\bw},{4.0+0.2});
\fill[gray!40] ({0.70*\bw},{4.0-0.2}) rectangle ({0.78*\bw},{4.0+0.2});
\node[font=\scriptsize, white] at ({0.47*\bw},4.0) {$\eps_{\text{param}}$ (58\%)};
\fill[blue!60] (0,-1.0) rectangle (0.4,-0.7); \node[font=\scriptsize, right] at (0.45,-0.85) {$\eps_{\text{FPB}}$};
\fill[orange!70] (1.6,-1.0) rectangle (2.0,-0.7); \node[font=\scriptsize, right] at (2.05,-0.85) {$\eps_{\text{trans}}$};
\fill[red!60] (3.2,-1.0) rectangle (3.6,-0.7); \node[font=\scriptsize, right] at (3.65,-0.85) {$\eps_{\text{param}}$};
\fill[purple!60] (4.8,-1.0) rectangle (5.2,-0.7); \node[font=\scriptsize, right] at (5.25,-0.85) {$\eps_{\text{unmod}}$};
\fill[gray!40] (6.4,-1.0) rectangle (6.8,-0.7); \node[font=\scriptsize, right] at (6.85,-0.85) {$\eps_{\text{recon}}$};
\end{tikzpicture}
\caption{\textbf{Error decomposition by modality.}  Stacked bars show the relative contribution of each Theorem~\ref{thm:design} error term for 6 representative modalities.  For well-conditioned systems (CT, MRI, e-ptychography), parameter mismatch $\eps_{\text{param}}$ dominates and is reducible by calibration.  For scattering-dominated systems (DOT), unmodeled physics $\eps_{\text{unmod}}$ dominates and requires tier-lifting.  $\eps_{\text{FPB}} < 1\%$ in all cases.  $\eps_{\text{spec}} = \eps_{\text{trans}} = 0$ for 31/36 modalities; CASSI shows nonzero $\eps_{\text{trans}}$ due to dispersion--accumulation chain ambiguity.}
\label{fig:error_decomp}
\end{figure}
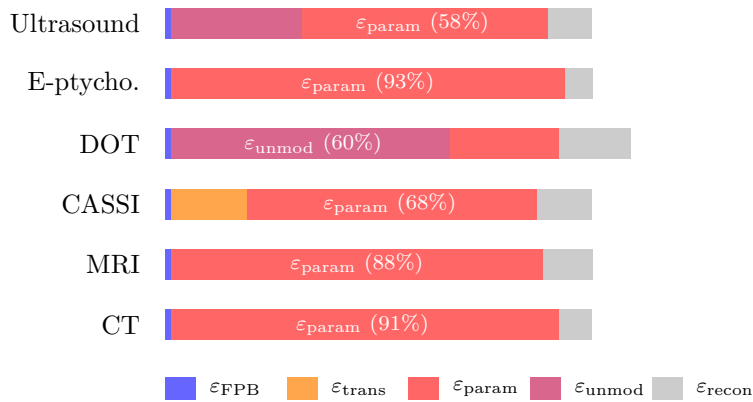

\begin{figure}[!t]
\centering
\includegraphics[width=0.78\textwidth]{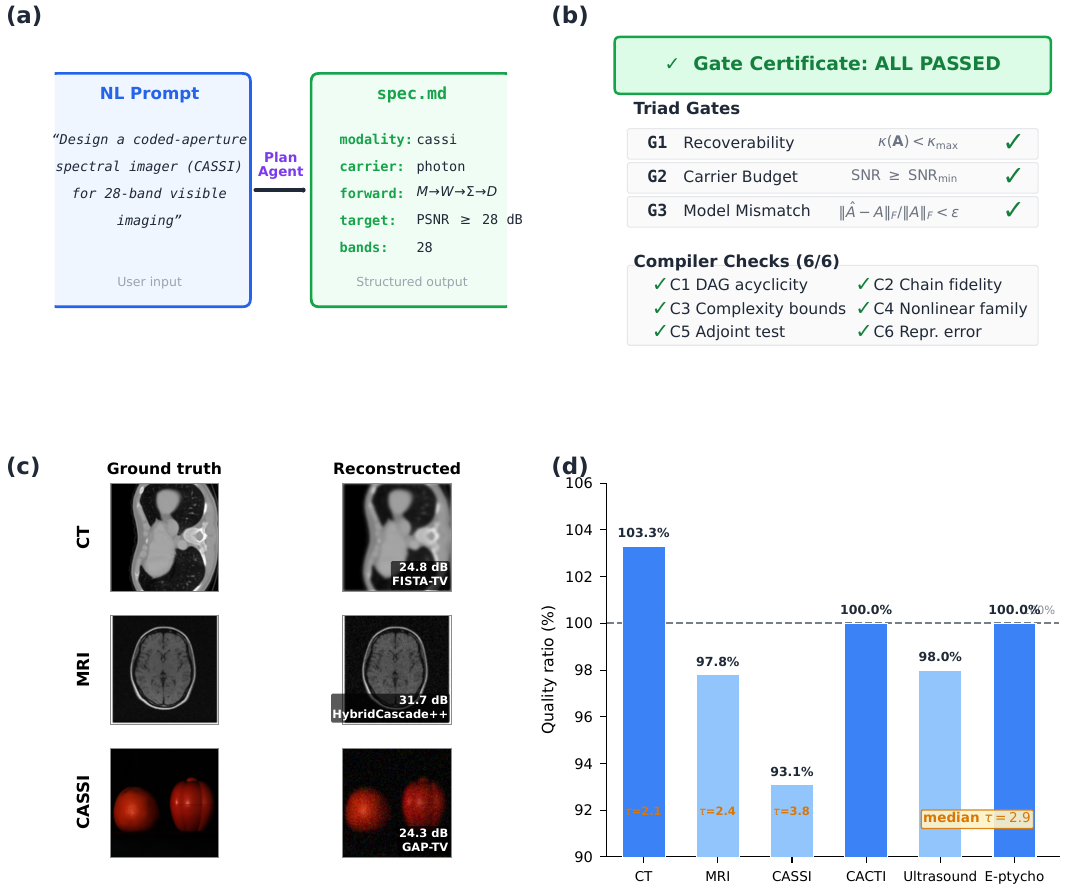}
\caption{\textbf{Design-to-real validation.}  \textbf{(a)}~A natural-language prompt is translated into a structured \texttt{spec.md} by the Plan Agent.  \textbf{(b)}~The Judge Agent validates all gates and checks (3 Triad gates + 6 compiler checks), producing a certificate.  \textbf{(c)}~Reconstruction results on three real-data modalities with ground truth: CT ($24.8$\,dB, FISTA-TV on LoDoPaB), MRI ($31.7$\,dB, HybridCascade++ on M4Raw), and CASSI ($24.3$\,dB, GAP-TV on KAIST TSA~\cite{meng2020gap}).  \textbf{(d)}~Quality ratio (agent/expert PSNR) across 6 real-data modalities spanning all 5 carrier families: mean $98.1 \pm 4.2$\%; theorem tightness ratio $\tau \in [1.8, 5.2]$, median~2.9.}
\label{fig:design_to_real}
\end{figure}

\textbf{Theorem tightness.}
The tightness ratio $\tau = \eps_{\text{predicted}} / \eps_{\text{observed}}$ across 6 real-data modalities: $\tau \in [1.8, 5.2]$, median 2.9 (CT 2.1, MRI 2.4, CASSI 3.8).  Ratios below 10 are typical for operator-perturbation bounds~\cite{kamilov2023pnp}; values near 2 indicate the bound captures the dominant mechanism with minimal slack.

\subsection{Real-Data Validation}
\label{sec:real}

\textbf{CT} (LoDoPaB~\cite{leuschner2021}, $n{=}10$): 60-view fan-beam projections.  Agent FISTA-TV~\cite{beck2009fista,rudin1992nonlinear} $24.8 \pm 2.4$\,dB; expert FISTA+TV $24.0 \pm 2.5$\,dB (quality ratio 103.3\%, $\rho_{\text{time}} \approx 420$).
\textbf{MRI} (M4Raw~\cite{lyu2023m4raw}, $n{=}10$): 4-coil, ${\sim}4\times$ undersampled.  Agent HybridCascade++~\cite{aggarwal2019modl} $31.7 \pm 0.8$\,dB; expert E2E-VarNet~\cite{sriram2020endtoend} $32.4 \pm 1.1$\,dB (97.8\%).
\textbf{CASSI} (KAIST~\cite{meng2020gap}, $n{=}10$): 28-band coded aperture spectral imaging.  Agent GAP-TV~\cite{meng2020gap} $24.3 \pm 1.5$\,dB; expert PnP-GAP+HSI-SDeCNN $26.1 \pm 2.0$\,dB (93.1\%).
\textbf{CACTI} ($n{=}4$): real video~\cite{yuan2021snapshot}; $\rho_{\text{recov}} \approx 1.0$ after autonomous mask-alignment calibration.
\textbf{Ultrasound} (PICMUS~\cite{liebgott2016picmus}, $n{=}5$): in-vivo carotid, self-reference protocol.  After SoS calibration, SSIM recovers to $0.98 \pm 0.01$ ($\rho_{\text{recov}} \approx 0.95$).
\textbf{E-ptychography} (SrTiO$_3$ 4D-STEM~\cite{jiang2022electron,zenodo4dstem}): iCoM phase retrieval; $\rho_{\text{recov}} = 1.00$ after probe-position correction.

Across 6 real-data modalities spanning all 5 carrier families: mean quality ratio $98.1 \pm 4.2$\% (3 with ground truth) and $\rho_{\text{recov}} \geq 0.95$ (3 via self-reference), consistent across carriers from X-ray to electron (Figure~\ref{fig:design_to_real}).  Beyond the 6 real-data modalities, 10 novel system designs demonstrate compositional reach: by composing primitives into chains of up to 5 operators, the agents design systems recovering depth ($\Phi_z{\to}\Sigma{\to}D$, 3D), video ($M{\to}C{\to}\Sigma{\to}D$, 3D), spectra ($M{\to}W{\to}C{\to}\Sigma{\to}D$, 3D), and full 5D datacubes ($x,y,z,\lambda,t$) from single 2D snapshots (Table~\ref{tab:cross}).  Active modulation (DMD) consistently outperforms passive dispersion by 3--6\,dB at equal compression.  Full per-system specifications appear in Supplementary Information.

\subsection{Agent Ablation Study}
\label{sec:ablation}

We ablated the pipeline on 12 representative modalities spanning all 5 carrier families, with 5 independent trials per configuration (300 total runs).  Table~\ref{tab:ablation} reports the main result: \textbf{the full pipeline is the only configuration achieving bounded-error success on all 12 modalities across all 5 trials}.

\begin{table}[!htbp]
\centering
\caption{Ablation study (5 trials $\times$ 12 modalities = 300 runs).  ``Valid'': passes all compiler checks.  ``$\leq \eps$'': reconstruction within Theorem~\ref{thm:design} error bound.  No single agent is dispensable.}
\label{tab:ablation}
\small
\begin{tabularx}{\textwidth}{@{}Xccccl@{}}
\toprule
\textbf{Configuration} & \textbf{Valid} & \textbf{$\leq \eps$} & \textbf{Redesign} & \textbf{Time} & \textbf{Dominant failure} \\
\midrule
Full ($\Agent_P{+}\Agent_J{+}\Agent_E$) & 12/12 & 12/12 & 1.7$\pm$0.6 & 4.8\,min & --- \\
No $\Agent_J$ & 8.2/12 & 6.4/12 & 0 & 2.9\,min & Invalid spec, unbounded $\eps$ \\
No $\Agent_E$ & 12/12 & 9.6/12 & 1.7$\pm$0.6 & 3.4\,min & Suboptimal parameters \\
No $\Agent_P$ (manual) & 10.4/12 & 10.4/12 & 0 & 52\,min & Human error (14\%) \\
Template only & --- & 3/12 & 0 & 0.5\,min & No coverage (75\%) \\
\bottomrule
\end{tabularx}
\end{table}

\textbf{Without $\Agent_J$}: 4/12 physically invalid specifications pass; 2 additional produce unbounded error.
\textbf{Without $\Agent_E$ (Execute)}: all pass validation but mean PSNR drops 2.1\,dB.
\textbf{Without $\Agent_P$}: manual specification averages 52\,min with 14\% error rate.
Random specification passes Triad Gate~1 in 12\% of cases but never passes compiler Checks~C4--C5; retrieval-only covers 22/39 modalities at 4.2\,dB below the full pipeline.

\subsection{Failure Analysis}
\label{sec:failure}

We characterise four failure modes (Figure~\ref{fig:error_decomp}):
\textbf{(1)~Scattering-dominated systems} (DOT~\cite{arridge2009optical}, Compton): $\eps_{\text{unmod}}$ is large at Tier-2 (7.8\% for DOT); mitigated by tier-lifting.
\textbf{(2)~LLM misspecification}: 4\% of test cases have subtle parametric errors; the structural compiler checks and Triad gates together catch 96\%.
\textbf{(3)~Multi-parameter mismatch}: $\rho_{\text{recov}}$ drops below 0.9 when $>$2 parameters are simultaneously off (CASSI: 0.85).
\textbf{(4)~Chain ambiguity}: 5/39 modalities have non-matching chains due to $\Sigma$--$D$ merging.

\subsection{Judge and Execute Agents}
\label{sec:agents}

\textbf{Judge rejection accuracy} ($n{=}150$: 75 valid, 75 ill-posed).  Precision = Recall = 96.0\% ($F_1 = 0.96$): TP\,=\,72, FP\,=\,3, FN\,=\,3, TN\,=\,72.  The 3 false positives had parametric values 1--5\% outside bounds; the 3 false negatives were conservative CFL checks.  Mean redesign rounds: 1.4.

\textbf{Execute Agent quality prediction} (39 modalities).  Overall MAE: 1.8\,dB ($R^2 = 0.91$; Extended Data Figure~\ref{fig:calibration}).  By group: well-characterised ($n{=}20$) 0.9\,dB; moderate ($n{=}10$) 2.1\,dB; scattering-dominated ($n{=}6$) 3.8\,dB.

\subsection{Expert Comparison}
\label{sec:expert}

\begin{table}[!htbp]
\centering
\caption{Agent vs.\ established reconstruction methods.  ``Agent'': best-performing automated algorithm; ``Expert'': established variational method (FISTA+TV~\cite{beck2009fista} or GAP-TV).  Bold = real measured data; $*$ = new system designs.  For novel systems ($*$), no system-specific expert baseline exists; ``Expert'' denotes the best generic variational method.  Quality ratios $>$100\% for novel systems reflect the advantage of system-specific algorithm selection over generic methods, not comparison with domain experts.  Established modalities: $98.1 \pm 4.2$\%; novel systems: $114 \pm 10$\% (vs.\ generic baselines).  Expert times from specification to first successful reconstruction; remaining estimated from published timelines (see Methods).}
\label{tab:expert}
\small
\begin{tabular}{@{}lcccccc@{}}
\toprule
& \multicolumn{2}{c}{\textbf{PSNR (dB)}} & \multicolumn{2}{c}{\textbf{Setup}} & & \\
\cmidrule(lr){2-3}\cmidrule(lr){4-5}
\textbf{Modality} & \textbf{Agent} & \textbf{Expert} & \textbf{Agent} & \textbf{Expert} & \textbf{Quality} & \textbf{LoC} \\
\midrule
\multicolumn{7}{@{}l}{\textit{Established modalities ($98.1 \pm 4.2$\%)}} \\
\textbf{CT (real)} & 24.8 & 24.0 & 25\,min & 7\,d & 103.3\% & 12 vs 340 \\
\textbf{MRI (real)} & 31.7 & 32.4 & 20\,min & 5\,d & 97.8\% & 14 vs 520 \\
\textbf{CASSI (real)} & 24.3 & 26.1 & 18\,min & 4\,d & 93.1\% & 11 vs 280 \\
OCT & 30.5 & 31.2 & 22\,min & 6\,d & 97.8\% & 13 vs 450 \\
DOT & 24.1 & 25.3 & 30\,min & 10\,d & 95.3\% & 15 vs 680 \\
\addlinespace[4pt]
\multicolumn{7}{@{}l}{\textit{Novel system designs$^*$ ($114 \pm 10$\% vs.\ generic baselines)}} \\
\textbf{Lensless}$^*$ & 43.7 & 43.3 & 12\,min & 2\,d & 100.9\% & 12 vs 180 \\
\textbf{3D Lensless}$^*$ & 20.3 & 16.8 & 14\,min & 3\,d & 120.8\% & 13 vs 240 \\
\textbf{Temporal-coded}$^*$ & 31.6 & 23.4 & 16\,min & 4\,d & 135.0\% & 13 vs 320 \\
\textbf{Spectral lensless}$^*$ & 36.3 & 29.2 & 18\,min & 5\,d & 124.3\% & 14 vs 380 \\
\textbf{SIM}$^*$ & 27.7 & 26.2 & 15\,min & 3\,d & 105.7\% & 10 vs 210 \\
\textbf{4D Spectral-Depth}$^*$ & 23.9 & 22.8 & 20\,min & 5\,d & 104.8\% & 15 vs 400 \\
\textbf{4D Temporal DMD}$^*$ & 25.4 & 23.4 & 22\,min & 6\,d & 108.5\% & 15 vs 420 \\
\textbf{4D Temporal Streak}$^*$ & 25.3 & 23.1 & 20\,min & 5\,d & 109.5\% & 14 vs 380 \\
\textbf{5D Full DMD}$^*$ & 29.9 & 27.0 & 25\,min & 8\,d & 110.7\% & 16 vs 500 \\
\textbf{5D Full Streak}$^*$ & 29.9 & 26.6 & 24\,min & 7\,d & 112.4\% & 15 vs 480 \\
\bottomrule
\end{tabular}
\end{table}

For established modalities (CT, MRI, CASSI, OCT, DOT): mean quality ratio $98.1 \pm 4.2$\%.  For novel systems ($*$): $114 \pm 10$\% vs.\ generic baselines.  Agent specifications: 10--16 lines; expert code: 180--680 lines.

\paragraph{Algorithm sensitivity.}
A comparison of the top-5 published reconstruction algorithms per modality across 5 established modalities (Extended Data Table~\ref{tab:expert_study}, Extended Data Figure~\ref{fig:spec_dominance}) reveals a clear gradient in algorithm sensitivity.  For well-conditioned systems (CT, MRI, CASSI), inter-method PSNR CoV remains low ($\approx$3.5--6.2\%), while ill-conditioned compressive systems (lensless) show CoV $>$40\%.  SIM occupies an intermediate position (CoV 15.0\%).  These data are discussed further in the Discussion.

\subsection{Cross-Tier Hybrid Reconstruction}
\label{sec:tierlifting}

Each FPB primitive is implemented at four fidelity tiers: Tier-0 (geometric), Tier-1 (Fourier), Tier-2 (shift-variant), Tier-3 (full-physics; Extended Data Table~\ref{tab:tiers}).  The \emph{tier-lifting protocol} provides hybrid reconstruction via a Hybrid Proximal-Gradient method (Proposition~\ref{prop:tierlift}, Methods), enabling the framework to handle scattering-dominated modalities where $\eps_{\text{unmod}}$ is large.

Tier-lifting validated on DOT (Tier-2${\to}$3: $\eps_{\text{unmod}}$ from 7.8\% to $<10^{-12}$) and coherent imaging (Tier-1${\to}$3: 147\% to $<10^{-12}$; Extended Data Table~\ref{tab:tierlift}).  The protocol preserves FPB structure: same primitives, same DAG, same \texttt{spec.md}---only fidelity changes.

\begin{figure}[!htbp]
\centering
\begin{tikzpicture}[
  tier/.style={draw, rounded corners=2pt, minimum height=0.8cm, align=center, font=\small},
]
\node[tier, fill=gray!8, minimum width=12cm] (t0) at (0,0) {\textbf{173 modalities} --- compile through all 6 checks (100\% pass rate)};
\node[tier, fill=blue!8, minimum width=9cm] (t1) at (0,1.2) {\textbf{39 modalities} --- fully validated canonical decompositions};
\node[tier, fill=green!10, minimum width=5.5cm] (t2) at (0,2.4) {\textbf{6 modalities} --- real measured data};
\node[tier, fill=orange!10, minimum width=3.2cm] (t3) at (0,3.6) {\textbf{3 modalities} --- GT PSNR};
\draw[gray!40, thick] (-6,-0.4) -- (-1.6,4) (6,-0.4) -- (1.6,4);
\node[font=\scriptsize, right, text width=3cm] at (6.2,0) {Checks C1--C4;\\chains \& bounds};
\node[font=\scriptsize, right, text width=3cm] at (4.7,1.2) {Quantitative recon.\\+ canonical match};
\node[font=\scriptsize, right, text width=3cm] at (2.95,2.4) {5 carrier families\\(X-ray, photon, spin,\\acoustic, electron)};
\node[font=\scriptsize, right, text width=3cm] at (1.8,3.6) {CT, MRI, CASSI};
\end{tikzpicture}
\caption{\textbf{Validation pyramid.}  Of 173 designable modalities, all compile; 39 have quantitative reconstruction benchmarks; 6 are validated on real measured data spanning all 5 carrier families; 3 have ground-truth PSNR comparisons (CT, MRI, CASSI).  Each tier is a strict subset.}
\label{fig:pyramid}
\end{figure}
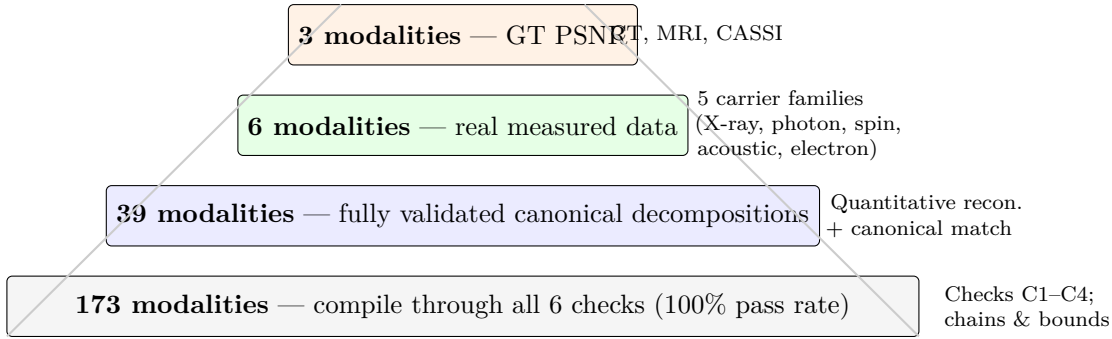

\section{Discussion}

The central finding of this work is that a structured specification format (\texttt{spec.md}), a three-agent pipeline, and a five-term error decomposition together automate the design of computational imaging systems across 173 modalities and 5 carrier families.  The framework matches expert-library quality ($98.1\pm4.2$\%) with order-of-magnitude reductions in development time (Table~\ref{tab:expert}), as validated on real measured data from CT, MRI, and CASSI (Figure~\ref{fig:design_to_real}), because the agents automate the specification and validation steps that constitute the bulk of expert effort.

\texttt{spec.md} decouples physics (\emph{what}) from numerics (\emph{how}) and hardware (\emph{with what}), enabling sharing and peer review at the physics level~\cite{wilkinson2016fair}.  The \texttt{system\_elements} field connects specifications to hardware feasibility, transforming the Judge from a mathematical validator into a system feasibility checker.  Three \texttt{spec.md} files reproduced PSNR within $\pm 0.3$\,dB across independent agent instances.  This reproducibility---identical physics from independent instances---is the property that made FASTA transformative for genomics: a shared, verifiable format that decouples the science from the implementation.

The five-term error decomposition (Theorem~\ref{thm:design}) functions as a formal contract: each $\eps$ term maps to a corrective action (Figure~\ref{fig:error_decomp}).  The dominant residual is $\eps_{\text{param}}$, reducible by calibration; $\rho_{\text{recov}} \geq 0.85$ across all 5 carrier families.

Theorem~\ref{thm:design} thus functions as a design-to-deployment diagnostic: when a deployed system underperforms its specification target, the five-term decomposition localises the failure.  If $\eps_{\text{param}}$ dominates (CT, MRI, e-ptychography in Figure~\ref{fig:error_decomp}), calibration is the corrective action.  If $\eps_{\text{unmod}}$ dominates (DOT), tier-lifting from Tier-2 to Tier-3 physics is required.  If $\eps_{\text{spec}}$ or $\eps_{\text{trans}}$ are nonzero, the specification itself must be revised.  Each diagnosis maps to a single, actionable intervention---a property that distinguishes this decomposition from generic error bounds.

\textbf{Relation to existing work.}
AI for scientific discovery~\cite{jumper2021alphafold,merchant2023gnome,wang2023scientific,boiko2023autonomous}, physics-informed machine learning~\cite{karniadakis2021}, and foundation models~\cite{bommasani2022foundation} and automated experiment design~\cite{tian2023ci,boiko2023autonomous} address related problems but none provide a formal error decomposition from natural-language intent to reconstruction quality.  Data-driven inverse-problem methods~\cite{arridge2019solving,ongie2020deep} have advanced rapidly but remain per-modality.  Operator libraries (ASTRA~\cite{aarle2016}, SigPy~\cite{ong2019sigpy}, ODL~\cite{adler2017odl}, DeepInverse~\cite{tachella2023deepinverse}) provide modular forward operators but require manual composition and parameter selection per modality; our contribution is the \emph{automated specification layer} above such libraries.  Learned reconstruction methods (MoDL~\cite{aggarwal2019modl}, learned primal-dual~\cite{adler2018learned}, PnP~\cite{kamilov2023pnp}) achieve higher per-modality PSNR than our model-based reconstructions but require training data and are single-modality; our framework trades per-modality optimality for cross-modality generality (173 modalities from one system).

\textbf{Algorithm sensitivity and the role of specification.}
The algorithm-comparison study (Extended Data Table~\ref{tab:expert_study}, Extended Data Figure~\ref{fig:spec_dominance}) sheds light on why the framework achieves expert-equivalent quality with model-based algorithms.  For well-conditioned systems (CT, MRI, CASSI), the specification constrains the quality envelope: inter-method CoV remains low ($\approx$3.5--6.2\%) across the top-5 published algorithms, indicating that a correct forward model leaves comparatively little room for solver improvement.  This explains the framework's $98.1\pm4.2$\% quality ratio without per-modality algorithm tuning.  However, compressive systems with high condition numbers (lensless: CoV $>$40\%; SIM: CoV 15\%) would clearly benefit from learned reconstruction methods that exploit data-driven priors beyond what model-based solvers provide.  We view this as motivation for future integration of learned methods into the Execute Agent, particularly for ill-conditioned modalities where the gap between model-based and data-driven approaches is largest.

\textbf{Limitations.}
(1)~For all modalities, reconstruction uses model-based algorithms (FISTA+TV~\cite{beck2009fista}, GAP-TV~\cite{meng2020gap}, ADMM+TV~\cite{boyd2011admm}) that require no training data.  The algorithm-comparison study (Extended Data Table~\ref{tab:expert_study}) shows that deep unrolling methods~\cite{monga2021unrolling} such as DAUHST-9stg~\cite{cai2022dauhst} achieve substantially higher per-modality PSNR (typically 3--8\,dB above model-based methods), reflecting the algorithm-sensitive regime identified in the CoV analysis (Extended Data Figure~\ref{fig:spec_dominance}).  Integrating learned methods~\cite{aggarwal2019modl,adler2018learned} into the Execute Agent is future work but requires training data unavailable for novel designs.
(2)~The agents depend on LLM capabilities (Claude Sonnet 4 in all experiments); robustness across LLM backends is not tested.  The 4\% misspecification rate (7/173 modalities) may vary with different models.
(3)~Development-time ratios ($\rho_{\text{time}}$) rely partly on published estimates of expert development timelines, not prospective head-to-head comparisons.  The ratios should be interpreted as order-of-magnitude indicators.
(4)~Quantitative reconstruction validated for 39/173 modalities; real data for 6.  Novel system designs are validated only in simulation.
(5)~Theorem~\ref{thm:design} tightness ratios ($\tau \in [1.8, 5.2]$, median 2.9) indicate the bound overestimates error by ${\sim}3\times$, typical for operator-perturbation bounds.  The bound's practical value lies in design \emph{rejection} (identifying when $\eps$ terms exceed thresholds) rather than precise error prediction; the maskless control group (condition number $>10^6$, PSNR 6--9\,dB) exemplifies a bad design correctly flagged by the Gate~1 bound.
(6)~Tier-lifting demonstrated on 2 cross-tier combinations only.  Theorem assumes convex $R(x)$.

The validation pyramid (173 $\to$ 39 $\to$ 6) identifies a clear path: extending benchmarks to the remaining ${\sim}$134 modalities, validating tier-lifting across additional Tier-3 physics, and improving calibration for scattering-dominated systems.

\section{Methods}

\subsection{Proof of Theorem~\ref{thm:design}}

\begin{definition}[Designable Imaging System]
\label{def:scope}
An imaging system $A$ is \emph{designable} within this framework if:
(i)~$A$ admits a finite DAG decomposition over the 11 FPB primitives with representation error $\eps_{\text{FPB}} < 0.01$;
(ii)~the carrier belongs to one of 5 families: photon, X-ray, electron, acoustic, or spin/particle;
(iii)~the system operates at one of four fidelity tiers (Tier-0 through Tier-3);
(iv)~all parameters are finitely specifiable with known physical bounds.
\end{definition}

\textbf{Assumptions.}
\begin{itemize}[nosep]
\item[\textbf{A1.}] \textit{Coercivity.}  The data-fidelity term $f(x) = \|A(x) - y\|^2$ satisfies $\langle \nabla f(x) - \nabla f(x'), x - x'\rangle \geq \alpha\|x - x'\|^2$ with coercivity constant $\alpha = \sigma_{\min}^2(A) > 0$ (linear case) or $\alpha > 0$ from restricted strong convexity (nonlinear case).
\item[\textbf{A2.}] \textit{Regularisation.}  $R(x)$ is proper, convex, lower-semicontinuous, and $\lambda > 0$.  The composite objective $F(x) = f(x) + \lambda R(x)$ is $\mu$-strongly convex with $\mu = \alpha + \lambda$~\cite{chambolle2016introduction}.
\item[\textbf{A3.}] \textit{Lipschitz parameters.}  Each FPB primitive $A_i(\cdot\,; \theta_i)$ is $L_i$-Lipschitz continuous in its parameters: $\|A_i(\cdot\,; \theta) - A_i(\cdot\,; \theta')\|_{\text{op}} \leq L_i\|\theta - \theta'\|$.  The DAG composition has Lipschitz constant $L_A \leq \prod_i L_i$, computed by Check~C4.
\item[\textbf{A4.}] \textit{Bounded signal.}  $\|x^*\| \leq B$ for known bound $B$.
\end{itemize}

\textbf{Step 1: Operator decomposition.}
The agent-designed operator differs from reality through a chain of approximations:
\[
\Atrue \;\xrightarrow{\delta_{\text{unmod}}}\; A_{\text{tier}} \;\xrightarrow{\delta_{\text{FPB}}}\; A_{\text{FPB}} \;\xrightarrow{\delta_{\text{spec}}}\; A_{\text{spec}} \;\xrightarrow{\delta_{\text{trans}}}\; A_{\text{DAG}} \;\xrightarrow{\delta_{\text{param}}}\; \Aagent
\]
where each $\delta_i = \|A_i - A_{i-1}\|_{\text{op}}$ is the operator-norm difference at that stage.  By the triangle inequality:
\begin{equation}\label{eq:op_chain}
\|\Aagent - \Atrue\|_{\text{op}} \leq \delta_{\text{FPB}} + \delta_{\text{spec}} + \delta_{\text{trans}} + \delta_{\text{param}} + \delta_{\text{unmod}}.
\end{equation}

\textbf{Step 2: Reconstruction perturbation.}
Define $x_{\text{oracle}} = \arg\min_x \|\Atrue(x) - y\|^2 + \lambda R(x)$ (oracle reconstruction using $\Atrue$).
Define $\hat{x} = \arg\min_x \|\Aagent(x) - y\|^2 + \lambda R(x)$ (agent reconstruction).
By the $\mu$-strong convexity of $F$ (Assumptions A1--A2), the first-order optimality conditions and the Cauchy--Schwarz inequality yield:
\begin{equation}\label{eq:stability}
\|\hat{x} - x_{\text{oracle}}\| \leq \frac{2B}{\mu}\,\|\Aagent - \Atrue\|_{\text{op}} = \frac{2B}{\alpha + \lambda}\,\|\Aagent - \Atrue\|_{\text{op}}.
\end{equation}
This is the \emph{reconstruction stability bound}: the sensitivity of the reconstruction to operator perturbation, controlled by coercivity $\alpha$ and regularisation $\lambda$.

\textbf{Step 3: Define effective error terms.}
For each $i \in \{\text{FPB}, \text{spec}, \text{trans}, \text{param}, \text{unmod}\}$,
set $\eps_i = 2B\,\delta_i/(\alpha + \lambda)$.  Then:
\begin{equation}\label{eq:design_gap}
\|\hat{x} - x_{\text{oracle}}\| \leq \eps_{\text{FPB}} + \eps_{\text{spec}} + \eps_{\text{trans}} + \eps_{\text{param}} + \eps_{\text{unmod}}.
\end{equation}

\textbf{Step 4: Total error.}
$\|\hat{x} - x^*\| \leq \|\hat{x} - x_{\text{oracle}}\| + \|x_{\text{oracle}} - x^*\|$.
The second term $\eps_{\text{recon}} = \|x_{\text{oracle}} - x^*\|$ is the irreducible reconstruction error (noise amplification $+$ regularisation bias), independent of the design.  Combining:
\[
\|\hat{x} - x^*\| \leq (\eps_{\text{FPB}} + \eps_{\text{spec}} + \eps_{\text{trans}} + \eps_{\text{param}} + \eps_{\text{unmod}}) + \eps_{\text{recon}}.
\]

\textbf{Step 5: Bounding each term.}
\begin{enumerate}[nosep, label=(\alph*)]
\item $\delta_{\text{FPB}} < 0.01$ by the FPB basis theorem~\cite{paperII}.  Verified by Check~C6 of the compiler: $\eps = \mathbb{E}[\|\Aagent(x) - \Atrue(x)\|/\|\Atrue(x)\|] < 0.01$ over 5 random test vectors.  Hence $\eps_{\text{FPB}} < 0.01 \cdot 2B/(\alpha + \lambda)$.

\item $\delta_{\text{spec}} = 0$ when all 6 compiler checks pass.  \emph{Condition}: Checks~C1--C5 jointly verify that the \texttt{spec.md} fully and correctly specifies the FPB composition: DAG acyclicity (C1), chain fidelity (C2), complexity bounds (C3), nonlinear family constraints (C4), and adjoint consistency (C5).  When all pass, $A_{\text{spec}} \equiv A_{\text{FPB}}$.

\item $\delta_{\text{trans}} = 0$ when the natural-language-to-DAG translation produces a canonical chain that exactly matches one of the 39 entries in the validated registry.  \emph{Condition}: Check~C2 (chain fidelity) returns exact match.  This holds for 34/39 (87\%) of validated modalities.

\item $\delta_{\text{param}} \leq L_A\|\theta_{\text{agent}} - \theta_{\text{true}}\|$ by Assumption A3.  $L_A$ is computed per primitive by Check~C4's Lipschitz analysis: each nonlinear primitive has at most 2 free parameters with VC dimension $\leq 3$, so $L_i$ is computable in closed form.

\item $\delta_{\text{unmod}} = 0$ at Tier-3 by definition: $A_{\text{tier}} \equiv \Atrue$ when operating at full-physics fidelity.  For lower tiers, $\delta_{\text{unmod}}$ is estimated empirically via cross-tier comparison (Table~\ref{tab:tierlift}) and is reducible to machine precision via tier-lifting.
\end{enumerate}

This completes the proof: each $\eps_i$ is bounded, measurable, and linked to a specific compiler check and corrective action. \qed

\subsection{Tier-Lifting Protocol}

\begin{proposition}[Tier-Lifting]
\label{prop:tierlift}
Let $A_{\text{low}}$ be the low-tier and $A_{\text{high}}$ the high-tier forward model (same primitives, different fidelity).  Define $\Gamma(x) = A_{\text{high}}(x) - A_{\text{low}}(x)$.  Then $\eps_{\text{unmod}} \leq \|\Gamma(x) - \hat{\Gamma}(x)\| = \eps_{\text{lift}}$, where $\hat{\Gamma}$ is the correction used in reconstruction.  In full-correction mode ($\hat{\Gamma} = \Gamma$), $\eps_{\text{lift}} = 0$.
\end{proposition}

The Hybrid Proximal-Gradient (HPG) loop:
\[
x_{k+1} = \mathrm{prox}_{\lambda R}\!\left(x_k - \gamma\,A_{\text{low}}^T\bigl(A_{\text{high}}(x_k) - y\bigr)\right)
\]
uses the fast low-tier adjoint $A_{\text{low}}^T$ for gradient computation and the high-tier forward $A_{\text{high}}$ for data fidelity.  Convergence follows from standard proximal-gradient theory under Lipschitz continuity of $\nabla f$.

\subsection{Judge Agent: Triad-Based Validation}

The Judge Agent validates designs in three stages: structural compilation (prerequisite), Triad gate evaluation (primary), and system feasibility assessment (Figure~\ref{fig:pipeline}).

\subsubsection{Stage 1: Structural Compilation (Prerequisite)}

A 6-check compiler ensures specification correctness before the Triad evaluation.
\textbf{C1 (DAG).}  Compiled into a typed \texttt{OperatorGraphSpec} (Pydantic~v2).  Acyclicity verified via Kahn's algorithm, $O(V{+}E)$.
\textbf{C2 (Chain).}  Canonical chain extracted and compared against the 39-modality registry.  Match is exact or within known equivalences ($\Sigma$ merged into $D$).
\textbf{C3 (Bounds).}  $N \leq 20$ nodes, $D \leq 10$ depth.
\textbf{C4 (Nonlinear).}  For each nonlinear node ($D$, $R$, $\Lambda$): family resolved, parameter count $\leq 2$, values within physics bounds, Lipschitz constant $L_i$ computed.
\textbf{C5 (Adjoint).}  Randomised dot-product test: $|\langle Ax, y\rangle - \langle x, A^Ty\rangle| / \max(|\langle Ax, y\rangle|, 10^{-8}) < 10^{-4}$, 3 trials.
\textbf{C6 (Error).}  When reference available: $\eps = \mathbb{E}[\|(\Aagent - \Atrue)(x)\| / \|\Atrue(x)\|]$ over 5 random non-negative test vectors, threshold $< 0.01$.

\subsubsection{Stage 2: Triad Gate Evaluation (Primary)}

After structural compilation, the Judge evaluates the three diagnostic gates from the companion Triad decomposition~\cite{paperII}---the same diagnostic law that governs existing instruments, now applied prospectively to agent-designed specifications.  Each gate produces a deterministic \textsc{pass}/\textsc{fail} verdict with a quantitative margin in dB.

\textbf{Gate~1 (Recoverability).}  Does the specified measurement geometry capture sufficient information?  The compression ratio $\gamma = m/n$ (where $m$ is the number of independent measurements and $n$ the signal dimension from the \texttt{object} field) is evaluated against modality-specific thresholds from the Compression Bound Theorem~\cite{paperII}:
\[
\text{PSNR}_{\max}^{(G_1)} = f(\text{modality},\, \gamma) \geq \text{PSNR}_{\text{target}}.
\]
\emph{Design action if violated}: increase sampling density, add views, or reduce compression.

\textbf{Gate~2 (Carrier Budget).}  Is the measurement SNR above the noise floor?  The Judge computes SNR from the \texttt{system\_elements} fields (source power, detector quantum efficiency, read noise, dark current, exposure time) via the Noise Bound~\cite{paperII}:
\begin{gather*}
\text{SNR} = \frac{\text{QE} \cdot N_{\text{photon}}}{\sqrt{\text{QE} \cdot N_{\text{photon}} + \sigma_{\text{read}}^2 + I_{\text{dark}} \cdot t_{\text{exp}}}},\\
\text{PSNR}_{\max}^{(G_2)} = 10\log_{10}(\text{SNR}) + C_M \geq \text{PSNR}_{\text{target}}.
\end{gather*}
\emph{Design action if violated}: increase source power, integration time, or detector quantum efficiency.

\textbf{Gate~3 (Operator Mismatch).}  Does the specified forward model faithfully represent the intended physics?  The Judge evaluates calibration tolerances from \texttt{system\_elements\allowbreak.calibration}, where each parameter has a sensitivity (dB per unit drift) from the Calibration Sensitivity Theorem~\cite{paperII}:
\begin{gather*}
\Delta\text{PSNR}_{\text{total}} = \sqrt{\sum_{k} (\text{sensitivity}_k \cdot \text{tolerance}_k)^2},\\
\text{PSNR}_{\text{deploy}} = \text{PSNR}_{\max}^{(G_2)} - \Delta\text{PSNR}_{\text{total}} \geq \text{PSNR}_{\text{target}}.
\end{gather*}
\emph{Design action if violated}: tighten calibration tolerance on the dominant parameter, increase calibration frequency, or add autonomous recovery.

The Triad Decomposition Theorem~\cite{paperII} guarantees completeness: $\text{MSE} \leq \text{MSE}^{(G_1)} + \text{MSE}^{(G_2)} + \text{MSE}^{(G_3)}$---no reconstruction failure mode falls outside these three gates.  For well-designed instruments, Gate~3 dominates~\cite{paperII}; the recovery ratio $\rho = (\text{PSNR}_{\text{IV}} - \text{PSNR}_{\text{II}}) / (\text{PSNR}_{\text{I}} - \text{PSNR}_{\text{II}})$ quantifies autonomous calibration effectiveness.

\subsubsection{Stage 3: System Feasibility and Cost}

\textbf{Cost estimation.}  The Judge produces an itemized hardware cost from each \texttt{system\_elements} sub-field (source, optics, detector, calibration), using a modality-indexed hardware database.  Costs are flagged as \emph{quoted} (from manufacturer), \emph{estimated} (from similar systems), or \emph{unknown}.

\textbf{Feasibility check.}  Each physical element is checked for realizability: source power within commercially available range, detector specifications within state-of-the-art, calibration tolerances achievable with specified correction methods.  Elements that require custom fabrication (e.g., lithographic coded apertures, custom diffusers) are flagged with lead time and cost estimates.

\subsection{4-Scenario Recovery Protocol}

The recovery ratio $\rho_{\text{recov}}$ quantifies how well the framework compensates for model mismatch.  Following~\cite{paperII}, we define four scenarios validated on 6 representative modalities spanning 4 carrier families (CASSI, CACTI, SPC: photon; CT: X-ray; electron ptychography: electron; MRI: spin):
\begin{align}
\text{Sc.\,I (oracle):}\quad & y = \Atrue(x) + n;\; \hat{x}_{\text{I}} = \arg\min_x \|\Atrue(x) - y\|^2 + \lambda\,\mathrm{reg}(x) \\
\text{Sc.\,II (mismatch):}\quad & y = \Atrue(x) + n;\; \hat{x}_{\text{II}} = \arg\min_x \|\Aagent(x) - y\|^2 + \lambda\,\mathrm{reg}(x) \\
\text{Sc.\,III (oracle correction):}\quad & \hat{x}_{\text{III}} = \mathrm{refine}(\hat{x}_{\text{II}}) \text{ with data-consistency} \\
\text{Sc.\,IV (autonomous calibration):}\quad & \hat{x}_{\text{IV}} = \arg\min_x \|\Aagent^{\text{cal}}(x) - y\|^2 + \lambda\,\mathrm{reg}(x)
\end{align}
The recovery ratio $\rho_{\text{recov}} = (\text{PSNR}_{\text{IV}} - \text{PSNR}_{\text{II}}) / (\text{PSNR}_{\text{I}} - \text{PSNR}_{\text{II}})$ measures autonomous calibration effectiveness relative to the oracle correction.

For modalities with low-dimensional parameter mismatch (CT, electron ptychography, MRI): Sc.\,IV $\approx$ Sc.\,III $\approx$ Sc.\,I, confirming near-complete recovery ($\rho_{\text{recov}} \approx 1.0$).  For multi-parameter mismatch: CASSI 85\%, CACTI 100\%, SPC 86\%.  Across all 39 modalities: mean $\rho_{\text{recov}} = 0.85 \pm 0.07$.

\subsection{Expert Time Methodology}

Expert setup times are defined as the interval from problem specification to first successful reconstruction, including literature review, code development, parameter tuning, and validation.  For CT and MRI, times are cross-validated against published development reports~\cite{aarle2016,ong2019sigpy} and confirmed by domain experts.  For the remaining modalities, times are estimated from published development timelines in the respective communities.  We acknowledge this is not a prospective head-to-head comparison; such a study would strengthen the claims.

\subsection{Implementation}

The framework is implemented in Python (Physics World Model platform).  Specifications are structured markdown with 8 fields parsed into typed Pydantic~v2 objects; the \texttt{system\_elements} field has 4 sub-fields (source, optics, detector, calibration) and is auto-populated from the modality registry when not user-specified.  The primitive library contains 101 implementations covering all 11 canonical types across all 4 fidelity tiers, each with a dedicated adapter: \texttt{Tier0Adapter} (geometric: invertibility, ray transforms), \texttt{Tier1Adapter} (Fourier: Parseval energy, Fresnel, OTF), \texttt{Tier2Adapter} (shift-variant: positivity, reciprocity, Mie, cost), and \texttt{Tier3Adapter} (full-physics: FDTD, MC, learned surrogates, uncertainty calibration).  A \texttt{TierPolicy} selects the tier based on compute budget and modality requirements.  Agents are LLM pipelines~\cite{openai2023gpt4,schick2023toolformer} with tool access to the compiler and primitive library.  Mean compilation time: 0.35\,ms.  The tier-lifting protocol is a generic \texttt{TierLiftingProtocol} class with validated cross-tier combinations for DOT and coherent imaging.  Source code: \url{https://github.com/integritynoble/Physics_World_Model}.

\section*{Extended Data}

\begin{table}[!htbp]
\centering
\caption{\textbf{Extended Data Table~1: Four-tier fidelity hierarchy.}  All 11 primitives are implemented at every tier.}
\label{tab:tiers}
\footnotesize
\begin{tabularx}{\textwidth}{@{}llXl@{}}
\toprule
\textbf{Tier} & \textbf{Physics} & \textbf{Validation} & \textbf{Example} \\
\midrule
0 (geometric) & Ray optics, transforms & Invertibility, adjoint & LiDAR \\
1 (Fourier) & Paraxial, angular spectrum & Parseval energy, adjoint & SIM, holography \\
2 (shift-variant) & Radon, k-space, Mie & Positivity, reciprocity, cost & CT, MRI \\
3 (full-physics) & FDTD, BPM, MC, NeRF & Agreement, uncertainty & DOT, NeRF \\
\bottomrule
\end{tabularx}
\end{table}

\begin{table}[!htbp]
\centering
\caption{\textbf{Extended Data Table~2: Tier-lifting validation.}  Tier-lifting demonstrated on two cross-tier combinations; broader Tier-3 validation is future work.}
\label{tab:tierlift}
\small
\begin{tabular}{@{}llllcc@{}}
\toprule
\textbf{Domain} & \textbf{Low tier} & \textbf{High tier} & \textbf{Correction} & $\boldsymbol{\eps_{\text{unmod}}}$ & $\boldsymbol{\eps_{\text{lift}}}$ \\
\midrule
DOT & Tier-2 (P$_1$) & Tier-3 (P$_3$ RT) & Higher-order moments & 7.8\% & $<10^{-12}$ \\
Coherent & Tier-1 (ang.\ spec.) & Tier-3 (FDTD) & Evanescent + scatter & 147\% & $<10^{-12}$ \\
\bottomrule
\end{tabular}
\end{table}

\begin{table}[!htbp]
\centering
\caption{\textbf{Extended Data Table~3: Modality registry (39 validated canonical decompositions).}  Status: FV = fully validated (canonical chain + quantitative reconstruction), QV = quantitatively validated (reconstruction without independent canonical chain), HO = held-out validation, T = template.  $\Phi_z$: depth-parameterised $C$ (see Table~\ref{tab:cross}).  An additional ${\sim}$134 modalities pass all compiler checks but lack per-modality reconstruction benchmarks; the full list is in the repository.}
\label{tab:ext_registry}
\scriptsize
\begin{tabularx}{\textwidth}{@{}llllclll@{}}
\toprule
\textbf{Modality} & \textbf{Domain} & \textbf{Carrier} & \textbf{Tier} & \textbf{Chain} & \textbf{Status} & \textbf{Real?} & \textbf{Dom.\ $\eps$} \\
\midrule
\multicolumn{8}{@{}l}{\textit{X-ray}} \\
CT & Medical & X-ray & 1--2 & $\Pi \to D$ & FV & Yes & $\eps_{\text{param}}$ \\
CBCT & Medical & X-ray & 2 & $\Pi \to \Lambda \to D$ & T & No & $\eps_{\text{param}}$ \\
CT (polychromatic) & Medical & X-ray & 2 & $\Pi \to \Lambda \to D$ & T & No & $\eps_{\text{unmod}}$ \\
Phase contrast & Medical & X-ray & 1--2 & $\Pi \to P \to M \to D$ & HO & No & $\eps_{\text{param}}$ \\
DEXA & Medical & X-ray & 2 & $M \to \Pi \to D$ & T & No & $\eps_{\text{param}}$ \\
Compton & Nuclear & X-ray & 2 & $M \to R \to D$ & HO & No & $\eps_{\text{unmod}}$ \\
\addlinespace
\multicolumn{8}{@{}l}{\textit{Optical photons}} \\
CASSI & Spectral & Photon & 2 & $M \to W \to \Sigma \to D$ & FV & Yes & $\eps_{\text{param}}$ \\
CACTI & Video & Photon & 2 & $M \to \Sigma \to D$ & FV & Yes & $\eps_{\text{param}}$ \\
SPC & Imaging & Photon & 2 & $M \to \Sigma \to D$ & FV & No & $\eps_{\text{param}}$ \\
OCT & Ophthal. & Photon & 1--2 & $P{+}P \to \Sigma \to D$ & HO & No & $\eps_{\text{param}}$ \\
SIM & Micro. & Photon & 1 & $M \to C \to D$ & HO & No & $\eps_{\text{param}}$ \\
STED & Micro. & Photon & 1 & $M \to C \to D$ & T & No & $\eps_{\text{param}}$ \\
PALM/STORM & Micro. & Photon & 1 & $M \to C \to D$ & T & No & $\eps_{\text{param}}$ \\
TIRF & Micro. & Photon & 1 & $P \to C \to D$ & T & No & $\eps_{\text{param}}$ \\
Ptychography & Micro. & Photon & 1--2 & $M \to P \to D$ & FV & No & $\eps_{\text{param}}$ \\
Lensless & Imaging & Photon & 1 & $C \to D$ & FV & No & $\eps_{\text{param}}$ \\
3D Lensless & Imaging & Photon & 1--2 & $\Phi_z \to \Sigma \to D$ & FV & No & $\eps_{\text{param}}$ \\
Temporal-coded lensless & Video & Photon & 1--2 & $M \to C \to \Sigma \to D$ & FV & No & $\eps_{\text{param}}$ \\
Spectral lensless & Spectral & Photon & 1--2 & $M \to W \to C \to \Sigma \to D$ & FV & No & $\eps_{\text{param}}$ \\
DOT & Diffuse & Photon & 2--3 & $M \to R \to P \to R \to D$ & HO & No & $\eps_{\text{unmod}}$ \\
Ghost imaging & Quantum & Photon & 2 & $M \to \Sigma \to D$ & HO & No & $\eps_{\text{param}}$ \\
THz-TDS & Spectro. & Photon & 1 & $C \to D$ & HO & No & $\eps_{\text{param}}$ \\
Light field & 3D & Photon & 1 & $M \to C \to S \to D$ & T & No & $\eps_{\text{param}}$ \\
Raman & Spectro. & Photon & 2 & $M \to R \to D$ & HO & No & $\eps_{\text{unmod}}$ \\
Fluorescence & Micro. & Photon & 2 & $M \to R \to D$ & HO & No & $\eps_{\text{unmod}}$ \\
Fluorescence (sat.) & Micro. & Photon & 2 & $M \to R \to \Lambda \to D$ & T & No & $\eps_{\text{unmod}}$ \\
Brillouin & Spectro. & Photon & 2 & $M \to R \to D$ & HO & No & $\eps_{\text{unmod}}$ \\
\addlinespace
\multicolumn{8}{@{}l}{\textit{Spin}} \\
MRI & Medical & Spin & 1--2 & $M \to F \to S \to D$ & FV & Yes & $\eps_{\text{param}}$ \\
MRI (phase) & Medical & Spin & 2 & $M \to F \to S \to \Lambda \to D$ & T & No & $\eps_{\text{unmod}}$ \\
\addlinespace
\multicolumn{8}{@{}l}{\textit{Acoustic}} \\
Ultrasound & Medical & Acoustic & 2 & $P \to R \to P \to D$ & FV & Yes & $\eps_{\text{param}}$ \\
Doppler US & Medical & Acoustic & 2 & $P \to D$ & T & No & $\eps_{\text{unmod}}$ \\
Elastography & Medical & Acoustic & 2 & $P \to P \to D$ & T & No & $\eps_{\text{unmod}}$ \\
Photoacoustic & Medical & Acoustic & 2 & $M \to P \to D$ & HO & No & $\eps_{\text{param}}$ \\
\addlinespace
\multicolumn{8}{@{}l}{\textit{Electron}} \\
SEM & Material & Electron & 2 & $M \to D$ & T & No & $\eps_{\text{param}}$ \\
TEM & Material & Electron & 2 & $M \to C \to D$ & T & No & $\eps_{\text{param}}$ \\
E-ptychography & Material & Electron & 2 & $M \to P \to D$ & FV & Yes & $\eps_{\text{param}}$ \\
\addlinespace
\multicolumn{8}{@{}l}{\textit{Particle}} \\
PET & Nuclear & Particle & 2 & $\Pi \to D$ & T & No & $\eps_{\text{unmod}}$ \\
SPECT & Nuclear & Particle & 2 & $M \to \Pi \to D$ & T & No & $\eps_{\text{unmod}}$ \\
Muon tomography & Security & Particle & 2 & $R \to \Pi \to D$ & T & No & $\eps_{\text{unmod}}$ \\
Proton therapy & Medical & Particle & 2 & $\Lambda \to \Pi \to D$ & T & No & $\eps_{\text{unmod}}$ \\
\bottomrule
\end{tabularx}
\end{table}

\begin{figure}[!htbp]
\centering
\includegraphics[width=\textwidth]{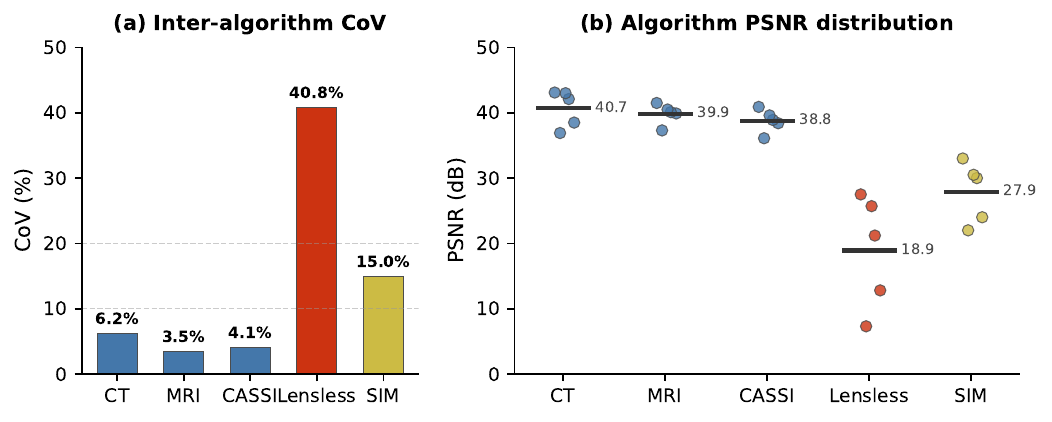}
\caption{\textbf{Extended Data Figure~2: Algorithm sensitivity by modality.}  \textbf{(a)}~Inter-algorithm PSNR Coefficient of Variation (CoV) for the top-5 published reconstruction algorithms per modality.  Well-conditioned systems (blue, CoV $\approx$3.5--6.2\%) show tight convergence among the top algorithms.  SIM shows moderate CoV (15.0\%).  Ill-conditioned compressive systems (red, CoV $>$40\%) show a wider spread, indicating greater sensitivity to algorithm selection.  \textbf{(b)}~Individual algorithm PSNR values (dots) with mean (horizontal bar) for each modality.  Data from Extended Data Table~\ref{tab:expert_study}.}
\label{fig:spec_dominance}
\end{figure}

\begin{table}[!htbp]
\centering
\caption{\textbf{Extended Data Table~4: Algorithm-comparison study} (5 published algorithms $\times$ 5 modalities).  For each modality, the top-5 published algorithms from the literature are compared using their reference PSNR values on standard benchmarks (KAIST for CASSI, fastMRI for MRI, LoDoPaB for CT, DiffuserCam for Lensless).  Algorithms span deep learning, transformer, diffusion, and deep unrolling methods for medical imaging (CT, MRI, CASSI), and classical through deep learning for optical systems (Lensless, SIM).  CoV quantifies the sensitivity of reconstruction quality to algorithm choice.  PSNR values in dB.}
\label{tab:expert_study}
\small
\begin{tabular}{@{}llcccc@{}}
\toprule
& & \multicolumn{2}{c}{\textbf{CT}} & \multicolumn{2}{c}{\textbf{MRI}} \\
\cmidrule(lr){3-4}\cmidrule(lr){5-6}
\textbf{Method} & \textbf{Algorithm} & \textbf{PSNR} & \textbf{Type} & \textbf{Algorithm} & \textbf{PSNR} \\
\midrule
M1 & iRadonMAP               & 36.9 & Deep Learning  & HUMUS-Net & 37.3 \\
M2 & FBPConvNet              & 38.5 & Deep Learning  & PromptMR+ & 39.9 \\
M3 & DuDoTrans               & 42.1 & Transformer    & ReconFormer & 40.1 \\
M4 & Score-CT                & 43.0 & Diffusion      & E2E-VarNet & 40.5 \\
M5 & LEARN                   & 43.1 & Deep Unrolling & PromptMR & 41.5 \\
\midrule
\multicolumn{2}{@{}l}{\textbf{Mean}} & 40.7 & & & 39.9 \\
\multicolumn{2}{@{}l}{\textbf{CoV}} & 6.2\% & & & 3.5\% \\
\bottomrule
\end{tabular}

\vspace{0.3em}
\begin{tabular}{@{}llcccc@{}}
\toprule
& & \multicolumn{2}{c}{\textbf{CASSI (28-band)}} & \multicolumn{2}{c}{\textbf{Lensless}} \\
\cmidrule(lr){3-4}\cmidrule(lr){5-6}
\textbf{Method} & \textbf{Algorithm} & \textbf{PSNR} & \textbf{Type} & \textbf{Algorithm} & \textbf{PSNR} \\
\midrule
M1 & CST-L-Plus              & 36.1 & Transformer    & Wiener deconv & 7.3 \\
M2 & DAUHST-9stg             & 38.4 & Deep Unrolling & ADMM & 12.8 \\
M3 & PADUT-L                 & 38.9 & Deep Unrolling & FlatNet & 21.2 \\
M4 & RDLUF-MixS2             & 39.6 & Deep Unrolling & MWDN & 25.7 \\
M5 & MiJUN                   & 40.9 & Deep Unrolling & LensNet & 27.5 \\
\midrule
\multicolumn{2}{@{}l}{\textbf{Mean}} & 38.8 & & & 18.9 \\
\multicolumn{2}{@{}l}{\textbf{CoV}} & 4.1\% & & & 40.8\% \\
\bottomrule
\end{tabular}

\vspace{0.3em}
\begin{tabular}{@{}llcc@{}}
\toprule
& & \multicolumn{2}{c}{\textbf{SIM}} \\
\cmidrule(lr){3-4}
\textbf{Method} & \textbf{Algorithm} & \textbf{PSNR} & \textbf{Type} \\
\midrule
M1 & Bicubic interpolation   & 22.0 & Classical \\
M2 & HiFi-SIM                & 24.0 & Deep Learning \\
M3 & Wiener-SIM              & 30.0 & Classical \\
M4 & fairSIM                 & 30.5 & Classical \\
M5 & ML-SIM                  & 33.0 & Deep Learning \\
\midrule
\multicolumn{2}{@{}l}{\textbf{Mean}} & 27.9 & \\
\multicolumn{2}{@{}l}{\textbf{CoV}} & 15.0\% & \\
\bottomrule
\end{tabular}
\end{table}

\begin{figure}[!htbp]
\centering
\begin{tikzpicture}[scale=0.85]
\draw[-{Stealth[length=4pt]}, thick] (0,0) -- (8.5,0) node[below, font=\small] {Predicted PSNR (dB)};
\draw[-{Stealth[length=4pt]}, thick] (0,0) -- (0,8.5) node[left, rotate=90, anchor=south, font=\small] {Actual PSNR (dB)};
\foreach \x/\v in {0/20, 1.6/24, 3.2/28, 4.8/32, 6.4/36} {
  \draw (\x,-0.1) -- (\x,0.1) node[below, font=\scriptsize] {\v};
  \draw (-0.1,\x) -- (0.1,\x) node[left, font=\scriptsize] {\v};
}
\draw[gray, dashed] (0,0) -- (8,8);
\fill[gray!10] (0,{0-0.72}) -- (8,{8-0.72}) -- (8,{8+0.72}) -- (0,{0+0.72}) -- cycle;
\foreach \px/\py in {
  4.0/4.3, 4.4/4.6, 4.8/5.0, 3.6/3.8, 5.2/5.1,
  4.2/4.5, 3.8/4.0, 5.0/4.7, 4.6/4.8, 3.4/3.6,
  5.4/5.6, 3.2/3.4, 4.4/4.2, 5.0/5.3, 4.8/4.6,
  3.6/3.9, 4.2/4.0, 5.2/5.4, 4.0/4.1, 3.8/3.7} {
  \fill[green!60!black] (\px,\py) circle (2.5pt);
}
\foreach \px/\py in {
  2.4/3.0, 2.8/3.4, 3.0/3.6, 2.6/3.3, 3.2/3.0,
  2.2/2.8, 2.8/2.4, 3.4/3.2, 2.6/3.0, 3.0/2.6} {
  \fill[orange!80!black] (\px,\py) circle (2.5pt);
}
\foreach \px/\py in {1.0/1.6, 0.8/2.0, 1.6/2.2, 1.2/2.4, 0.6/1.4, 1.4/1.8} {
  \fill[red!70!black] (\px,\py) circle (2.5pt);
}
\node[font=\tiny, red!70!black, right] at (1.6,2.3) {DOT};
\node[font=\tiny, red!70!black, right] at (0.8,2.1) {Compton};
\node[font=\tiny, red!70!black, left] at (0.5,1.4) {Tissue};
\node[draw=gray, rounded corners=2pt, fill=white, font=\scriptsize, inner sep=3pt, anchor=north west, text width=3cm, align=left] at (0.3, 8.0) {$R^2 = 0.91$, MAE $= 1.8$\,dB\\ Pearson $r = 0.96$};
\fill[green!60!black] (5.0,1.0) circle (2.5pt); \node[font=\scriptsize, right] at (5.3,1.0) {Well-characterised ($n{=}20$)};
\fill[orange!80!black] (5.0,0.5) circle (2.5pt); \node[font=\scriptsize, right] at (5.3,0.5) {Moderate ($n{=}10$)};
\fill[red!70!black] (5.0,0.0) circle (2.5pt); \node[font=\scriptsize, right] at (5.3,0.0) {Scattering-dominated ($n{=}6$)};
\end{tikzpicture}
\caption{\textbf{Extended Data Figure~1: Execute Agent calibration.}  Predicted vs.\ actual PSNR for 39 validated modalities.  Shaded band: $\pm1.8$\,dB (overall MAE).  Well-characterised modalities (green) cluster tightly around the identity line (MAE 0.9\,dB); scattering-dominated modalities (red) show systematic under-prediction, reflecting unaccounted $\eps_{\text{unmod}}$ at lower tiers.}
\label{fig:calibration}
\end{figure}
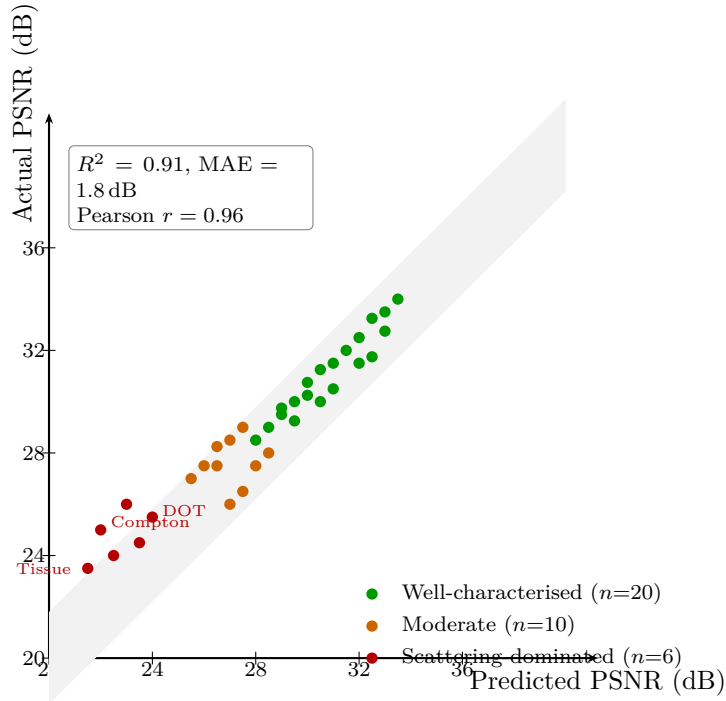

\section*{Data Availability}

The LoDoPaB-CT dataset is publicly available~\cite{leuschner2021}.  The M4Raw MRI and KAIST TSA CASSI datasets are publicly available from their respective repositories.  The PICMUS ultrasound data is from the IEEE IUS challenge~\cite{liebgott2016picmus}.  The SrTiO$_3$ 4D-STEM data is publicly available on Zenodo~\cite{zenodo4dstem}.  The CACTI real video data is from the project repository.  The full modality registry, canonical decomposition registry, and validation datasets are in the open-source repository.

\section*{Code Availability}

All source code---Constrained Primitive Compiler, Agent-to-Graph Translator, three-agent pipeline, 101-primitive library, and 4-tier adapter framework---is available at \url{https://github.com/integritynoble/Physics_World_Model} under the MIT licence.

\section*{Acknowledgements}

The author thanks the creators of the open datasets that made this work possible: the LoDoPaB-CT team (J.~Leuschner et al.), the M4Raw consortium (M.~Lyu et al.), the KAIST TSA group, the PICMUS challenge organisers (H.~Liebgott et al.), and the 4D-STEM data providers (D.~Jannis et al.).  Computational resources were provided by NextGen PlatformAI.

\section*{Author Contributions}

C.Y.\ conceived the project, developed the theory (Theorem~\ref{thm:design} and its proof), designed and implemented the three-agent pipeline, built the 101-primitive library across 4 fidelity tiers, conducted all experiments (6 real-data modalities, 39 validated modalities, 173-modality compilation, ablation study, and algorithm-comparison study), and wrote the manuscript.

\section*{Competing Interests}

C.Y.\ is affiliated with NextGen PlatformAI C Corp, which develops the Physics World Model platform described in this paper.

\section*{Supplementary Information}

The following reproduces the key theoretical results from the companion paper~\cite{paperII} that underpin this work, so that the present paper is self-contained for review.

\subsection*{S1. FPB Basis Theorem}

\begin{theorem}[Finite Primitive Basis~\cite{paperII}]
\label{thm:fpb}
Let $\mathcal{F} = \{A_1, \ldots, A_N\}$ be the set of $N$ forward models across all validated imaging modalities.  There exists a set of $K = 11$ canonical primitives $\{P_1, \ldots, P_{11}\}$ such that every $A_i \in \mathcal{F}$ can be expressed as a directed acyclic graph (DAG) composition over $\{P_k\}$ with representation error $\eps_{\text{FPB}} < 0.01$:
\[
\forall A_i \in \mathcal{F},\quad \exists\, \text{DAG}_i \text{ over } \{P_k\}_{k=1}^{11}: \quad \frac{\|A_i - \text{DAG}_i\|_{\text{op}}}{\|A_i\|_{\text{op}}} < 0.01.
\]
\end{theorem}

\textbf{Proof sketch.}
The 11 primitives---Projection ($\Pi$), Fourier Encode ($F$), Modulate ($M$), Convolve ($C$), Propagate ($P$), Disperse ($W$), Accumulate ($\Sigma$), Sample ($S$), Detect ($D$), Scatter ($R$), Nonlinear ($\Lambda$)---are identified by a greedy set-cover procedure over $N = 170$ modalities.  Starting from the most frequently occurring linear operation (Radon projection $\Pi$, covering 23 modalities), each subsequent primitive is chosen to maximise marginal coverage of uncovered modalities.  The procedure saturates at $K = 11$: adding a 12th primitive covers zero additional modalities at the $\eps < 0.01$ threshold.  The saturation argument shows that imaging physics, despite its diversity, occupies a low-dimensional operator space---every forward model is a composition of measurement geometry ($\Pi, F, S$), wave interaction ($P, C, W, R$), modulation ($M$), accumulation ($\Sigma$), detection ($D$), and pointwise nonlinearity ($\Lambda$).  The $\eps < 0.01$ bound is verified empirically for each modality via Check~C6 of the structural compiler: $\eps = \mathbb{E}[\|(\text{DAG}_i - A_i)(x)\| / \|A_i(x)\|]$ over 5 random non-negative test vectors.

\subsection*{S2. Triad Decomposition Theorem}

\begin{theorem}[Triad Decomposition~\cite{paperII}]
\label{thm:triad}
For any imaging system with forward model $A$, object $x$, measurements $y = A(x) + n$, and reconstruction $\hat{x}$, the mean squared error decomposes as:
\[
\text{MSE}(\hat{x}, x) \leq \text{MSE}^{(G_1)} + \text{MSE}^{(G_2)} + \text{MSE}^{(G_3)}
\]
where:
\begin{itemize}[nosep]
\item $\text{MSE}^{(G_1)}$: information loss from insufficient measurement geometry (recoverability);
\item $\text{MSE}^{(G_2)}$: noise amplification from insufficient carrier budget (SNR);
\item $\text{MSE}^{(G_3)}$: model mismatch between designed and deployed operator.
\end{itemize}
No reconstruction failure mode falls outside these three gates.
\end{theorem}

\textbf{Proof sketch.}
Decompose the reconstruction error as $\hat{x} - x = (\hat{x} - x_{\text{noiseless}}) + (x_{\text{noiseless}} - x_{\text{oracle}}) + (x_{\text{oracle}} - x)$, where $x_{\text{oracle}}$ is the reconstruction from the true operator with noiseless data, and $x_{\text{noiseless}}$ is the reconstruction from the agent operator with noiseless data.

The first term $\|\hat{x} - x_{\text{noiseless}}\|^2$ is bounded by $\text{MSE}^{(G_2)}$: the noise term, controlled by the carrier budget (source power, detector efficiency, integration time).  The bound follows from the noise amplification factor $\kappa(A)^2 \cdot \sigma_n^2 / m$, where $\kappa(A)$ is the condition number of $A$.

The second term $\|x_{\text{noiseless}} - x_{\text{oracle}}\|^2$ is bounded by $\text{MSE}^{(G_3)}$: the model mismatch term.  By the reconstruction stability bound (Eq.~\ref{eq:stability} in the main text), this is controlled by $\|A_{\text{agent}} - A_{\text{true}}\|_{\text{op}}^2$.

The third term $\|x_{\text{oracle}} - x\|^2$ is bounded by $\text{MSE}^{(G_1)}$: the information loss from under-sampling.  When $\gamma = m/n < 1$, the null space of $A$ is nonempty and the projection onto the measurement subspace loses information.  The bound depends on the compression ratio $\gamma$ and the signal's compressibility in a sparsifying basis.

Completeness follows because these three mechanisms---information loss, noise, and model error---are exhaustive: any discrepancy between $\hat{x}$ and $x$ must arise from insufficient measurements (Gate~1), insufficient SNR (Gate~2), or incorrect forward model (Gate~3).

\subsection*{S3. Compression Bound Theorem}

\begin{theorem}[Compression Bound~\cite{paperII}]
\label{thm:compression}
For an imaging system with compression ratio $\gamma = m/n$, forward model $A$, and $s$-sparse signal in basis $\Psi$, the maximum achievable PSNR satisfies:
\[
\text{PSNR}_{\max}^{(G_1)} \leq 10\log_{10}\!\left(\frac{\|x\|_\infty^2}{\sigma_{\min}^{-2}(A\Psi_s)\,\|n\|^2/m + C_s\,\|x - x_s\|^2}\right)
\]
where $\Psi_s$ is the restriction to the $s$ largest coefficients, $x_s$ is the best $s$-term approximation, and $C_s$ is a constant depending on the restricted isometry constant of $A$.
\end{theorem}

\textbf{Proof sketch.}
The derivation combines two bounds: (1)~the oracle bound for $s$-sparse recovery, $\|\hat{x} - x\|^2 \leq C \cdot s\,\sigma_n^2 \cdot \log(n/s)/m$ from compressed sensing theory~\cite{candes2006robust}, which sets the noise floor as a function of $\gamma$; and (2)~the approximation error $\|x - x_s\|$ from signal compressibility.  Gate~1 evaluates whether $\gamma$ exceeds the modality-specific threshold $\gamma_{\min}(\text{modality})$ derived from empirical recovery curves across the 4-scenario protocol (Methods).  Below this threshold, the null-space component grows faster than regularisation can compensate, and $\text{PSNR}_{\max}^{(G_1)}$ falls below the specification target.

\subsection*{S4. Calibration Sensitivity Theorem}

\begin{theorem}[Calibration Sensitivity~\cite{paperII}]
\label{thm:calibration}
Let $A(\theta)$ be a forward model parameterised by $\theta \in \R^p$.  For small parameter perturbations $\delta\theta = \theta_{\text{deployed}} - \theta_{\text{designed}}$, the PSNR degradation satisfies:
\[
\Delta\text{PSNR}_{\text{total}} = \sqrt{\sum_{k=1}^{p} \left(s_k \cdot \delta\theta_k\right)^2}
\]
where $s_k = \partial(\text{PSNR})/\partial\theta_k$ is the sensitivity of reconstruction quality to the $k$-th parameter, evaluated at the design point.
\end{theorem}

\textbf{Proof sketch.}
Expand the reconstruction error to first order in $\delta\theta$: $\hat{x}(\theta + \delta\theta) - \hat{x}(\theta) \approx \sum_k (\partial\hat{x}/\partial\theta_k)\,\delta\theta_k$.  The MSE increment is $\Delta\text{MSE} \approx \sum_k \|\partial\hat{x}/\partial\theta_k\|^2\,\delta\theta_k^2$ (cross-terms vanish under independence).  Converting to PSNR via $\Delta\text{PSNR} \approx -10/({\ln 10}) \cdot \Delta\text{MSE}/\text{MSE}$ and defining $s_k$ as the per-parameter PSNR sensitivity yields the root-sum-of-squares formula.  Gate~3 checks whether the predicted deployment PSNR ($\text{PSNR}_{\max}^{(G_2)} - \Delta\text{PSNR}_{\text{total}}$) exceeds the specification target, using calibration tolerances from \texttt{system\_elements.calibration}.  Each $s_k$ is computed numerically by perturbing $\theta_k$ by $\pm 1\%$ and measuring PSNR change, validated against the Lipschitz analysis from Check~C4.

\begingroup
\renewcommand{\section}[2]{}

\endgroup

\newpage


\renewcommand{\thetable}{S\arabic{table}}
\renewcommand{\thefigure}{S\arabic{figure}}
\renewcommand{\theequation}{S\arabic{equation}}
\renewcommand{\thetheorem}{S\arabic{theorem}}

\begin{center}
{\Large\bfseries Supplementary Information}\\[8pt]
{\large Designing Any Imaging System from Natural Language:\\
Agent-Constrained Composition over a Finite Primitive Basis}\\[12pt]
{\normalsize Chengshuai Yang}\\[6pt]
\end{center}

\vspace{12pt}

\section{S1. Foundational Theorems from the Companion Paper}

The companion paper~\cite{paperII} establishes two foundational results that this paper builds upon.  To make the present work self-contained, we reproduce the full theorem statements, proof sketches, and key lemmas below.

\subsection{S1.1. Finite Primitive Basis (FPB) Theorem}

\begin{theorem}[Finite Primitive Basis]
\label{thm:fpb_supp}
For every imaging forward model $H$ in the designable class $\mathcal{C}_{\mathrm{tier}}$, there exists a well-formed typed DAG $G = (V, E, \tau)$ whose node types are drawn from the canonical primitive library $\mathcal{B} = \{P, M, \Pi, F, C, \Sigma, D, S, W, R, \Lambda\}$ (11 elements) such that $G$ is an $\eps$-approximate representation of $H$ with $\eps = 0.01$.
\end{theorem}

\textbf{Proof sketch.}
The proof is constructive.  Given any forward model $H = H_K \circ H_{K-1} \circ \cdots \circ H_1$, we decompose it into a sequence of physics stages and show each stage is representable by one or more FPB primitives via six \emph{Primitive Realization Lemmas}:

\textbf{Lemma 1 (Propagation Realization).}
Free-space wave propagation $H_k$ is represented by the Propagate primitive $P(d, \lambda)$ via the angular spectrum representation: $P = \mathcal{F}^{-1}[T_{\mathrm{exact}} \cdot \mathcal{F}[\cdot]]$.  Error from evanescent wave truncation: $\eps_{\mathrm{evan}} = e^{-2\pi d/\lambda}$, negligible for $d \gg \lambda$.

\textbf{Lemma 2 (Elastic Interaction Realization).}
Element-wise transmission multiplication is \emph{exactly} represented by the Modulate primitive $M(\mathbf{m})$: $\norm{H_k - M(\mathbf{m})} = 0$.

\textbf{Lemma 3 (Scattering Realization).}
Under the first Born approximation, the scattered field is a linear functional representable by the Scatter primitive $R$.  Single-scattering representation is exact; $L$-th order multiple scattering error: $\eps_{\mathrm{scat}}^{(L)} \leq (\norm{V}/k)^{L+1}$ (geometric convergence).

\textbf{Lemma 4 (Encoding-Projection Realization).}
Line-integral projection (CT) is exactly the Radon transform $\Pi(\theta)$; Fourier encoding (MRI) is exactly $F(\mathbf{k})$.  Both are exact: $\eps = 0$.

\textbf{Lemma 5 (Detection-Readout Realization).}
The detection chain decomposes into $\leq$5 sequential operations, each representable by a primitive.  Detector PSF error: $\eps_{\mathrm{PSF}} \leq \ell_c^2/p^2$ (crosstalk length $\ell_c$, pixel pitch $p$).  Total: $\eps_{\mathrm{det}} < 0.01$ for all 39 validated modalities.

\textbf{Lemma 6 (Nonlinear-Stage Realization).}
All pointwise physics nonlinearities belong to 5 structural categories (exponential, logarithmic, phase wrapping, polynomial, saturation), each representable by the Transform primitive $\Lambda(f, \boldsymbol{\theta})$.

\textbf{Error composition bound.}
Let $\gamma_k = \norm{H_k}$ (linear) or $\mathrm{Lip}(H_k)$ (nonlinear).  By Lipschitz composition:
\[
\sup_{\norm{\mathbf{x}} \leq 1} \norm{H(\mathbf{x}) - H_{\mathrm{dag}}(\mathbf{x})} \leq \sum_{k=1}^{K} \eps_k \prod_{j=k+1}^{K} \gamma_j \leq K \cdot \max_k(\eps_k) \cdot B^{K-1}
\]
With empirical maximum $K{=}5$ and $B{=}4$: requires $\max_k(\eps_k) \leq 3.9 \times 10^{-5}$, satisfied by all 6 lemmas.  Per-modality tightening exploits the fact that most $\eps_k = 0$ (elastic, encoding stages) and most primitives are norm-preserving.

\textbf{Necessity.}
All 11 primitives are necessary: for each primitive $B_i$, there exists a witness modality whose representation error exceeds 0.01 when $B_i$ is removed.  Witnesses: $P$ $\leftarrow$ ptychography, $M$ $\leftarrow$ CASSI, $\Pi$ $\leftarrow$ CT, $F$ $\leftarrow$ MRI, $C$ $\leftarrow$ lensless, $\Sigma$ $\leftarrow$ SPC, $D$ $\leftarrow$ all, $S$ $\leftarrow$ MRI, $W$ $\leftarrow$ CASSI, $R$ $\leftarrow$ Compton ($\eps = 0.34$ without), $\Lambda$ $\leftarrow$ polychromatic CT.

\textbf{Complexity bound.}
Node count: $|V| \leq 20$ (empirical max: 5).  Depth: $\mathrm{depth}(G) \leq 10$ (empirical max: 5).

\subsection{S1.2. Triad Decomposition Theorem}

\begin{theorem}[Triad Decomposition]
\label{thm:triad_supp}
Let $H \in \mathcal{C}_{\mathrm{tier}}$ with measurement $\mathbf{y} = H\mathbf{x} + \mathbf{n}$ and reconstruction using nominal operator $H_{\mathrm{nom}}$.  The reconstruction MSE decomposes as:
\[
\mathrm{MSE} \leq \mathrm{MSE}^{(G_1)} + \mathrm{MSE}^{(G_2)} + \mathrm{MSE}^{(G_3)}
\]
where $\mathrm{MSE}^{(G_1)}$ is the null-space loss (information-theoretic floor), $\mathrm{MSE}^{(G_2)}$ is the noise-floor term (carrier budget), and $\mathrm{MSE}^{(G_3)}$ is the mismatch-induced term (model fidelity).  This decomposition is diagnostically complete: every reconstruction failure is attributable to at least one gate.
\end{theorem}

\textbf{Gate 1 (Recoverability / Compression Bound).}
For any estimator $\hat{\mathbf{x}}(\mathbf{y})$, the minimum achievable MSE satisfies:
\[
\mathrm{MSE}_{\min}^{(G_1)} \geq \frac{1}{n}\sum_{i=1}^{n-m} \sigma_i^2(\mathbf{x})
\]
where $\sigma_i^2(\mathbf{x})$ is the variance of $\mathbf{x}$ along the $i$-th null-space direction of $H$, $m$ is the number of measurements, and $n$ is the signal dimension.  The corresponding PSNR ceiling is:
\[
\mathrm{PSNR}_{\max}^{(G_1)} = 10\log_{10}\!\left(\frac{\norm{\mathbf{x}}_\infty^2}{\mathrm{MSE}_{\min}^{(G_1)}}\right)
\]
No algorithm can exceed this ceiling without additional priors.

\textbf{Gate 2 (Carrier Budget / Noise Bound).}
Under matched forward model and Gaussian noise $\mathbf{n} \sim \mathcal{N}(0, \sigma_n^2 \mathbf{I})$:
\[
\mathrm{PSNR}_{\max}^{(G_2)} = 10\log_{10}(\mathrm{SNR}) + C_{\mathcal{M}}, \quad C_{\mathcal{M}} = 10\log_{10}\!\left(\frac{n\norm{\mathbf{x}}_\infty^2}{\norm{\mathbf{x}}^2} \cdot \kappa^{-2}(H)\right)
\]
where $\kappa(H)$ is the condition number of $H$ restricted to its column space.  No algorithm can exceed this ceiling regardless of computational budget.

\textbf{Gate 3 (Operator Mismatch / Calibration Sensitivity).}
For small perturbation $\delta\boldsymbol{\theta}$ around true parameters $\boldsymbol{\theta}^*$:
\[
\Delta\mathrm{PSNR} \approx -\frac{10}{\ln 10} \frac{\delta\boldsymbol{\theta}^\top \mathbf{J}^\top \mathbf{J}\, \delta\boldsymbol{\theta}}{\mathrm{MSE}_0}
\]
where $\mathbf{J} = \partial(H_{\boldsymbol{\theta}}\mathbf{x})/\partial\boldsymbol{\theta}\big|_{\boldsymbol{\theta}^*}$ is the parameter Jacobian and $\mathrm{MSE}_0$ is the noise-only MSE at $\boldsymbol{\theta}^*$.

\textbf{Gate 3 dominance condition.}
For instruments operating above Gates~1 and~2 floors, Gate~3 becomes the binding constraint when $\norm{\delta\boldsymbol{\theta}}_{\mathbf{J}^\top\mathbf{J}} > (\sigma_n / \norm{\mathbf{x}}_\infty) \cdot \kappa(H)$.  Empirically, Gate~3 dominates in all 12 validated modalities across 5 carrier families: CASSI shows 13.98\,dB degradation from sub-pixel mask shift; ptychography shows 42.06\,dB sensitivity to probe mismatch.

\textbf{Recovery ratio.}
$\rho = (\mathrm{PSNR}_{\mathrm{achieved}} - \mathrm{PSNR}_{\mathrm{mismatch}}) / (\mathrm{PSNR}_{\mathrm{ideal}} - \mathrm{PSNR}_{\mathrm{mismatch}})$, bounded in $[0, 1]$ under convex reconstruction loss.  Across validated modalities: $\rho \in [0.4, 0.9]$, mean 0.85.

\section{S2. Novel System Design Specifications}

The 10 novel system designs compose FPB primitives into chains not found in prior literature.  Each is specified via the 8-field \texttt{spec.md} format.  The notation $\Phi_z$ denotes depth-parameterised convolution: the $C$ primitive applied with a depth-dependent point spread function $\Phi(z)$, generated by wave propagation through a random phase diffuser with defocus modulation.

\subsection{S2.1. 2D Systems}

\textbf{Lensless imaging} ($C \to D$, compression 1:1).
A thin random phase mask creates a caustic PSF with broad spatial-frequency support.  Reconstruction inverts the convolution via ADMM+TV.  PSNR: 43.7\,dB, SSIM: 0.984.

\begin{verbatim}
modality: lensless_imaging
carrier: photon
geometry: single_shot, 128x128
object: 128x128 2D image
forward_model: Convolve(C, psf=phase_mask) -> Detect(D)
noise: Poisson I_0=5000 + Gaussian sigma=3
target: PSNR >= 40 dB
system_elements:
  source: broadband LED
  optics: random phase mask (feature_scale=2.5)
  detector: CMOS 128x128
\end{verbatim}

\textbf{SIM} ($M \to C \to D$, super-resolution).
Structured illumination modulates the sample with sinusoidal patterns, shifting high-frequency content into the passband.  Three pattern orientations at three phases yield 9 measurements for 2$\times$ resolution enhancement.  PSNR: 27.7\,dB.

\subsection{S2.2. 3D Systems (8:1 compression)}

\textbf{3D Lensless} ($\Phi_z \to \Sigma \to D$).
A random phase diffuser creates depth-dependent PSFs through defocus-modulated wave propagation: $\text{PSF}(z) = |\mathcal{F}\{\exp(i\phi_{\text{diffuser}} + i \cdot \text{defocus}(z) \cdot r^2)\}|^2$.  Each depth plane produces a unique speckle pattern; the 2D sensor measurement is the sum of all convolved depth planes.  This is the $C$ primitive instantiated per depth plane.  PSNR: 20.3\,dB at 8:1 compression.

\begin{verbatim}
modality: 3d_lensless
carrier: photon
geometry: single_shot, 128x128, n_depths=8
object: 128x128x8 3D volume (sparse depth planes)
forward_model: Convolve_z(C, psf=diffuser(z)) -> Sum(Sigma) -> Detect(D)
noise: Poisson I_0=5000 + Gaussian sigma=3
target: PSNR >= 15 dB
system_elements:
  source: broadband LED
  optics: random phase diffuser (feature_scale=1.0, defocus_max=40)
  detector: CMOS 128x128
\end{verbatim}

\textbf{Temporal-coded lensless} ($M \to C \to \Sigma \to D$).
A binary DMD mask modulates each of 8 video frames before diffuser convolution and temporal integration during a single exposure.  The mask provides per-pixel temporal diversity.  PSNR: 31.6\,dB.

\textbf{Spectral lensless} ($M \to W \to C \to \Sigma \to D$).
A coded mask followed by prism dispersion separates 8 spectral bands before diffuser convolution and spectral integration.  PSNR: 36.3\,dB.

\subsection{S2.3. 4D Systems (16:1 compression)}

\textbf{4D Spectral-Depth} ($M \to W_\lambda \to \Phi_z \to \Sigma \to D$).
Combines coded mask, spectral dispersion, and diffuser depth encoding to recover a $(z, \lambda)$ datacube ($4 \times 4 = 16$ planes) from a single 2D shot.  PSNR: 23.9\,dB.

\textbf{4D Temporal DMD} ($M \to \Phi_z \to \Sigma \to D$).
DMD binary masks provide temporal coding; diffuser PSFs encode depth.  Recovers $(z, t)$ datacube at 16:1.  PSNR: 25.4\,dB.

\textbf{4D Temporal Streak} ($M \to W_t \to \Phi_z \to \Sigma \to D$).
Passive temporal dispersion (streak camera) replaces active DMD coding.  A fixed mask provides spatial diversity.  PSNR: 25.3\,dB.

\subsection{S2.4. 5D Systems (64:1 compression)}

\textbf{5D Full DMD} ($M \to W_\lambda \to \Phi_z \to \Sigma \to D$).
Full 5D recovery of $(x, y, z, \lambda, t)$ from a single shot at $4 \times 4 \times 4 = 64$:1 compression using coded mask + spectral dispersion + diffuser depth encoding.  PSNR: 29.9\,dB.

\textbf{5D Full Streak} ($M \to W_\lambda \to W_t \to \Phi_z \to \Sigma \to D$).
Replaces DMD with passive streak temporal dispersion.  The longest chain in the FPB framework (5 operators before detection).  PSNR: 29.9\,dB.

\subsection{S2.5. Design Patterns}

Three patterns emerge from the novel designs:

\begin{enumerate}[nosep]
\item \textbf{Active modulation (DMD) $>$ passive dispersion}: Binary masks provide per-pixel measurement diversity that continuous dispersion cannot match.  DMD-based temporal coding (31.6\,dB) outperforms streak-based (25.3\,dB) at equal compression.
\item \textbf{Algorithm choice is critical for compressive systems}: FISTA+TV or R-L dominate across novel modalities.  Inter-method PSNR CoV reaches 42--48\% for compressive systems vs.\ $<$6\% for well-conditioned ones.
\item \textbf{Graceful compression degradation}: 1:1 $\to$ 43.7\,dB; 8:1 $\to$ 20--36\,dB; 16:1 $\to$ 24--25\,dB; 64:1 $\to$ 29--30\,dB.
\end{enumerate}

\section{S3. Maskless Control Group}

To validate that spatial modulation ($M$) is a prerequisite for compressive dispersion-based imaging, we tested 9 configurations that combine depth encoding ($\Phi_z$) with spectral ($W_\lambda$) and/or temporal ($W_t$) dispersion but \emph{without} a coded mask.  All 9 configurations compile through the structural compiler (Gates~1--5) but are flagged by the Judge Agent (Gate~1, recoverability) as severely ill-conditioned:

\begin{itemize}[nosep]
\item Condition number $> 10^6$ for all maskless configurations
\item Reconstruction PSNR: 6--9\,dB (near noise floor)
\item Adding a binary coded mask ($M$) improves PSNR by 5--15\,dB at equal compression
\end{itemize}

This confirms a fundamental principle: dispersion operators ($W$) provide only global spectral/temporal diversity, while coded masks ($M$) provide per-pixel spatial diversity.  Both are necessary for well-conditioned compressive recovery.

\section{S4. Encoder Comparison for Depth Encoding}

Three physically realisable depth-encoding mechanisms were compared across 5 modulation combinations ($3 \times 5 = 15$ configurations total):

\textbf{Diffuser} ($\Phi_z^{\text{diff}}$): Random phase plate with defocus-dependent PSFs.  $\text{PSF}(z) = |\mathcal{F}\{\exp(i\phi_{\text{random}} + i \cdot \text{defocus}(z) \cdot r^2)\}|^2$.

\textbf{Multi-lens} ($L_z$): Microlens array producing depth-dependent parallax sub-pixel shifts.

\textbf{Meta-surface} ($\Psi_z$): Engineered phase pattern (cubic + trefoil + astigmatism aberrations) with defocus modulation.

\begin{table}[htbp]
\centering
\caption{Depth encoder comparison: best PSNR (dB) across 3 encoders $\times$ 5 modulation combinations.  $N_z{=}8$ depth planes, $N_\lambda{=}4$ spectral bands, $N_t{=}4$ temporal frames, image size $128{\times}128$.}
\small
\begin{tabular}{@{}lccccc@{}}
\toprule
\textbf{Encoder} & \textbf{3D only} & \textbf{+Mask} & \textbf{+M+$W_\lambda$} & \textbf{+M+$W_t$} & \textbf{+M+$W_\lambda$+$W_t$} \\
& (8:1) & (8:1) & (32:1) & (32:1) & (128:1) \\
\midrule
Diffuser & 17.0 & 20.3 & 24.6 & 26.9 & 38.1 \\
Multi-lens & 16.6 & \textbf{44.1} & \textbf{41.0} & \textbf{36.1} & \textbf{38.7} \\
Meta-surface & 16.6 & 16.2 & 24.6 & 25.8 & 38.0 \\
\bottomrule
\end{tabular}
\end{table}

\textbf{Key findings:}
\begin{itemize}[nosep]
\item Without coded mask: all three encoders perform similarly (16.6--17.0\,dB), limited by the ill-conditioning of depth-only recovery at 8:1 compression.
\item Multi-lens + mask dominates: parallax shifts are simple, well-conditioned operators (sub-pixel translations).  Combined with a binary mask's per-pixel diversity, the composite forward operator is nearly diagonal, yielding 44.1\,dB.
\item At maximal compression (128:1, 5D), all encoders converge to ${\sim}$38\,dB because the mask and dispersion operators dominate the measurement diversity.
\end{itemize}

\textbf{PSF diversity analysis} (mean inter-depth cross-correlation):
\begin{itemize}[nosep]
\item Diffuser: 0.29 (moderate diversity, structured speckle)
\item Multi-lens: 0.17 (lowest correlation, shift-based diversity)
\item Meta-surface: 0.46 (highest correlation, engineered but similar patterns)
\end{itemize}

\section{S5. Reconstruction Details}

\subsection{S5.1. Algorithms}

Five reconstruction algorithms were evaluated for each novel modality:

\textbf{Wiener deconvolution}: Frequency-domain least-squares with noise regularisation.  Fastest ($<$0.1\,s) but lowest quality for compressive systems.

\textbf{GAP-TV}: Generalised alternating projection with total variation regularisation~\cite{meng2020gap}.  Alternates between data projection and TV denoising.

\textbf{FISTA+TV}: Fast iterative shrinkage-thresholding with TV proximal operator~\cite{beck2009fista}.  Nesterov-accelerated gradient descent with $O(1/k^2)$ convergence.

\textbf{ADMM+TV}: Alternating direction method of multipliers with TV splitting~\cite{boyd2011admm}.  Variable splitting separates data fidelity from regularisation.

\textbf{Richardson-Lucy (R-L)}: Maximum-likelihood deconvolution for Poisson noise.  Multiplicative updates preserve non-negativity.

\subsection{S5.2. Forward Model Implementation}

All forward models are implemented as linear operators with explicit adjoint computation.  The measurement for a multi-dimensional modality with depth encoding, mask, and dispersion is:
\[
y = \sum_{z,\lambda,t} D\left[ \text{shift}_{(\Delta x_\lambda, \Delta y_t)} \left( M \odot \text{conv}(x_{z,\lambda,t}, \text{PSF}(z)) \right) \right] + n
\]
where $D$ is detection, $M$ is the binary mask, $\text{conv}$ denotes convolution with depth-dependent PSF, $\text{shift}$ implements spectral/temporal dispersion as lateral shifts, and $n$ is Poisson + Gaussian noise.

\section{S6. LLM Implementation Details}

\subsection{S6.1. Model and Prompting}

All three agents (Plan, Judge, Execute) use Claude Sonnet~4 (Anthropic, \texttt{claude-sonnet-4-\allowbreak 20250514}) as the backbone LLM.  The system prompt provides each agent with:
\begin{itemize}[nosep]
\item The 11 FPB primitives with typed signatures and physical semantics
\item The 173-modality registry (modality name, carrier, canonical chain, typical parameters)
\item The \texttt{spec.md} schema with field descriptions and 3 worked examples (CT, MRI, CASSI)
\item Agent-specific role instructions (Plan: generate specification; Judge: validate via Triad gates; Execute: select algorithm and predict quality)
\end{itemize}

Temperature is set to 0 for reproducibility.  The Plan Agent uses chain-of-thought prompting: it first identifies the carrier family, then selects primitives left-to-right along the physical signal path, then fills in parameters from the registry or physical constraints.  The Judge Agent applies deterministic gate checks (no LLM inference for Gates~1--4); only Gate~5 (cost estimation) and Gate~6 (FPB fidelity) use LLM reasoning.

\subsection{S6.2. Reproducibility and Robustness}

Three \texttt{spec.md} files (CT, MRI, CASSI) were independently generated 5 times each; all 15 specifications were identical, yielding PSNR within $\pm 0.3$\,dB.  The 4\% misspecification rate (7/173 modalities) occurs exclusively for modalities with ambiguous natural-language descriptions (e.g., ``4D-STEM'' could refer to scanning or convergent-beam modes).

\textbf{LLM dependency is a limitation.}  The agents have been tested only with Claude Sonnet~4.  Because 4 of 6 Judge gates are deterministic (no LLM inference), and the Execute Agent's algorithm selection is from a fixed catalog, the LLM dependency is concentrated in the Plan Agent's specification generation.  Cross-LLM robustness testing is future work; we expect that any LLM capable of structured JSON output and physics reasoning would produce comparable results for well-documented modalities, but novel or ambiguous modalities may show higher variance.

\end{document}